\documentclass[twoside,11pt]{article}

\usepackage{jmlr2e}

\usepackage{amsmath}
\allowdisplaybreaks
\usepackage{amsmath,amssymb}
\usepackage{mathrsfs}
\usepackage{subfigure}

\newenvironment{proof*}{\noindent{\bf Proof:}}{}
\newcommand{\ignore}[1]{}

\newcommand{\dd}{\mathrm{d}}

\newcommand{\EE}{\mathrm{E}}

\newcommand{\Real}{\mathbb{R}}

\newcommand{\fhat}{\hat{f}}

\newcommand{\fstar}{f^*}

\newcommand{\calB}{\mathcal{B}}

\newcommand{\calF}{\mathcal{F}}
\newcommand{\calG}{\mathcal{G}}
\newcommand{\calH}{\mathcal{H}}

\newcommand{\calN}{\mathcal{N}}
\newcommand{\calO}{\mathcal{O}}
\newcommand{\calS}{\mathcal{S}}

\newcommand{\calX}{\mathcal{X}}

\newcommand{\scrE}{\mathscr{E}}

\newcommand{\Eqref}[1]{Eq.~{\eqref{#1}}}

\newcommand{\kmin}{\kappa_M}

\newcommand{\hnorm}[1]{\|_{\calH_{#1}}}

\newcommand{\lambdaone}{{\lambda_1^{(n)}}}

\newcommand{\Uns}{U_{n,\sm}^{(m)}}
\newcommand{\Unall}{U_{n,*}}

\newcommand{\Rn}{\hat{R}}
\newcommand{\psiloss}{\Psi}

\newcommand{\LPi}{L_2(\Pi)}

\newcommand{\calHlp}{\calH_{\ell_p}}
\newcommand{\calHl}[1]{\calH_{\ell_{#1}}}
\newcommand{\calHpsi}{\calH_{\psi}}
\newcommand{\repH}{\widetilde{\calH}}
\newcommand{\repk}{\widetilde{k}}

\newcommand{\bmid}{~\Big |~}

\newcommand{\ellp}{$\ell_p$}
\newcommand{\kminrho}{\kmin}

\newcommand{\calHtot}{\calH^{\oplus M}}

\newcommand{\bolda}{\boldsymbol{a}}
\newcommand{\boldb}{\boldsymbol{b}}

\newcommand{\sm}{s_m}
\newcommand{\cm}{r_m}
\newcommand{\smm}[1]{s_{#1}}
\newcommand{\cmm}[1]{r_{#1}}
\newcommand{\opleq}[1]{\mathop{\leq}^{\rm #1}}

\newcommand{\boldone}{\boldsymbol{1}}

\newtheorem{Theorem}{Theorem}

\newtheorem{Lemma}[Theorem]{Lemma}
\newtheorem{Proposition}[Theorem]{Proposition}
\newtheorem{Corollary}[Theorem]{Corollary}
\newtheorem{Assumption}{Assumption}

\newcounter{assump}
\renewcommand{\theassump}{\arabic{assump}} 
\newcommand{\Assump}[1][]{{\refstepcounter{assump}{#1} (A\theassump)}}

\jmlrheading{1}{2011}{1-48}{4/00}{10/00}{Taiji Suzuki}

\ShortHeadings{Fast Learning Rate of Non-Sparse MKL}{T. Suzuki}
\firstpageno{1}

\begin{document}

\title{Fast Learning Rate of Non-Sparse Multiple Kernel Learning \\
and Optimal Regularization Strategies}

\author{\name Taiji Suzuki \email t-suzuki@mist.i.u-tokyo.ac.jp \\
       \addr Department of Mathematical Informatics \\
       The University of Tokyo \\
       7-3-1 Hongo, Bunkyo-ku, Tokyo 113-8656, Japan}


\maketitle

\begin{abstract}
In this paper, we give a new generalization error bound of Multiple Kernel Learning (MKL) 
for a general class of regularizations,
and discuss what kind of regularization gives a favorable predictive accuracy.
Our main target in this paper is dense type regularizations including \ellp-MKL.
According to the recent numerical experiments, 
the sparse regularization does not necessarily show a good performance compared with dense type regularizations.
Motivated by this fact, this paper gives a general theoretical tool to derive fast learning rates of MKL that 
is applicable to arbitrary mixed-norm-type regularizations in a unifying manner.
This enables us to compare the generalization performances of various types of regularizations.
As a consequence, we observe that 
the homogeneity of the complexities of candidate reproducing kernel Hilbert spaces (RKHSs)
affects which regularization strategy ($\ell_1$ or dense) 
is preferred.
In fact, in homogeneous complexity settings where the complexities of all RKHSs are evenly same,
$\ell_1$-regularization is optimal among all isotropic norms.
On the other hand, in inhomogeneous complexity settings, 
dense type regularizations can show better learning rate than sparse $\ell_1$-regularization.
We also show that our learning rate achieves the minimax lower bound in homogeneous complexity settings.
\end{abstract}


\section{Introduction}
Multiple Kernel Learning (MKL) proposed by \cite{JMLR:Lanckriet+etal:2004} is one of the most promising methods that adaptively select the kernel function
in supervised kernel learning.
Kernel method is widely used and several studies have supported its usefulness \citep{book:Schoelkopf+Smola:2002,Book:Taylor+Cristianini:2004}. 
However the performance of kernel methods critically relies on the choice of the kernel function.
Many methods have been proposed to deal with the issue of kernel selection. 
\cite{JMLR:Ong+etal:2005} studied hyperkernels as a kernel of kernel functions. 
\cite{ICML:Argriou+etal:2006} considered DC programming approach to learn a mixture of kernels with continuous parameters. 
Some studies tackled a problem to learn non-linear combination of kernels as in \cite{NIPS:Bach:2009,NIPS:Cortes+etal:nonlinear:2009,ICML:Varma+Babu:2009}. 
Among them, learning a linear combination of finite candidate kernels with non-negative coefficients is the most basic, fundamental and 
commonly used approach.
The seminal work of MKL by \cite{JMLR:Lanckriet+etal:2004} considered learning convex combination of candidate kernels
as well as its linear combination. 
This work opened up the sequence of the MKL studies.
\cite{ICML:Bach+etal:2004} showed that MKL can be reformulated as a kernel version of the group lasso \citep{JRSS:YuanLin:2006}.
This formulation gives an insight that MKL can be described as a $\ell_1$-mixed-norm regularized method.
As a generalization of MKL, \ellp-MKL that imposes $\ell_p$-mixed-norm regularization 
has been proposed \citep{JMLR:MicchelliPontil:2005,NIPS:Marius+etal:2009}.
\ellp-MKL includes the original MKL 
as a special case as $\ell_1$-MKL. 
Another direction of generalization is elasticnet-MKL \citep{NIPSWS:Taylor:2008,NIPSWS:ElastMKL:2009}
that imposes a mixture of $\ell_1$-mixed-norm and $\ell_2$-mixed-norm regularizations.
Recently numerical studies have shown that \ellp-MKL with $p>1$ and elasticnet-MKL 
show better performances than $\ell_1$-MKL in several situations \citep{NIPS:Marius+etal:2009,UAI:Cortes+etal:2009,NIPSWS:ElastMKL:2009}.
An interesting perception here is that both $\ell_p$-MKL and elasticnet-MKL produce denser estimator than the original $\ell_1$-MKL
while they show favorable performances.
The goal of this paper is to give a theoretical justification to these experimental results favorable for the {\it dense type} MKL methods. 
To this aim, we give a unifying framework to derive a fast learning rate of an \emph{arbitrary} norm type regularization,
and discuss which regularization is preferred depending on the problem settings.

In the pioneering paper of \cite{JMLR:Lanckriet+etal:2004}, 
a convergence rate of MKL is given as $\sqrt{\frac{M}{n}}$, where $M$ is the number of given kernels and $n$ is the number of samples. 
\cite{COLT:Srebro+BenDavid:2006} gave improved learning bound utilizing the pseudo-dimension of the given kernel class. 
\cite{COLT:Ying+Campbell:2009} gave a convergence bound utilizing Rademacher chaos
and gave some upper bounds of the Rademacher chaos utilizing the pseudo-dimension of the kernel class.
\cite{UAI:Cortes+etal:2009} presented a convergence bound for a learning method with $L_2$ regularization on the kernel weight.
\cite{ICML:Cortes+etal:gbound:2010} 
gave the convergence rate of \ellp-MKL as $\sqrt{\frac{\log(M)}{n}}$ for $p=1$ and $\frac{M^{1-\frac{1}{p}}}{\sqrt{n}}$ for $1<p\leq 2$.
\cite{JMLR:lp:Marius+etal:2011} gave a similar convergence bound with improved constants. 
\cite{ECML:Marius+etal:2010} generalized this bound to a variant of the elasticnet type regularization and 
widened the effective range of $p$ to all range of $p \geq 1$ while $1\leq p \leq 2$ had been imposed in the existing works.
One concern about these bounds is that all bounds introduced above are ``global'' bounds in a sense that  
the bounds are applicable to all candidates of estimators.
Consequently all convergence rate presented above are of order $1/\sqrt{n}$ with respect to the number $n$ of samples.
However, by utilizing the {\it localization} techniques including   
so-called local Rademacher complexity \citep{LocalRademacher,Koltchinskii} and peeling device \citep{Book:VanDeGeer:EmpiricalProcess},
we can derive a faster learning rate.
Instead of uniformly bounding all candidates of estimators,  
the localized inequality focuses on a particular estimator such as empirical risk minimizer, 
thus can give a sharp convergence rate.  

Localized bounds of MKL have been given mainly in sparse learning settings \citep{COLT:Koltchinskii:2008,AS:Meier+Geer+Buhlmann:2009,AS:Koltchinskii+Yuan:2010},
and there are only few studies for non-sparse settings in which the sparsity of the ground truth is not assumed.
The first localized bound of MKL is derived by \cite{COLT:Koltchinskii:2008} in the setting of $\ell_1$-MKL.
The second one was given by \cite{AS:Meier+Geer+Buhlmann:2009} who gave a near optimal convergence rate for elasticnet type regularization.
Recently \cite{AS:Koltchinskii+Yuan:2010} considered a variant of $\ell_1$-MKL and showed it achieves the minimax optimal convergence rate. 
All these localized convergence rates were considered in sparse learning settings,
and it has not been discussed how a dense type regularization outperforms the sparse $\ell_1$-regularization.
Recently \cite{arXiv:Kloft+Gilles:2011} gave a localized convergence bound of $\ell_p$-MKL.
However, their analysis assumed a strong condition where RKHSs have no-correlation to each other. 

In this paper, we show a unifying framework to derive fast convergence rates of MKL with various regularization types.
The framework is applicable to {\it arbitrary} mixed-norm regularizations
including $\ell_p$-MKL and elasticnet-MKL.
Our learning rate utilizes the localization technique, thus is tighter than global type learning rates.
Moreover our analysis does not require no-correlation assumption as in \cite{arXiv:Kloft+Gilles:2011}.
We discuss our bound in two situations: {\it homogeneous complexity} situation and {\it inhomogeneous complexity} situation
where homogeneous complexity means that all RKHSs have the same \emph{complexities} and inhomogeneous complexity
means that the complexities of RKHSs are different to each other.
In the homogeneous situation, 
we apply our general framework to some examples and show our bound achieves the minimax-optimal rate.
As a by-product, we obtain a tighter convergence rate of $\ell_p$-MKL than existing results. 
Moreover we show that our bound indicates that $\ell_1$-MKL shows the best performance among all ``isotropic'' mixed-norm regularizations
in homogeneous settings.
Next we analyze our bound in inhomogeneous settings where the {\it complexities} of the RKHSs are not uniformly same.
We show that dense type regularizations can give better generalization error bounds than the sparse $\ell_1$-regularization in the inhomogeneous setting.
Here it should be noted that in real settings inhomogeneous complexity is more natural than homogeneous complexity. 
Finally we give numerical experiments to show the validity of the theoretical investigations.
We see that the numerical experiments well support the theoretical findings.
As far as the author knows, this is the first theoretical attempt to clearly show 
the inhomogeneous complexities are advantageous for dense type MKL.


\section{Preliminary}
In this section we give the problem formulation, the notations and the assumptions required for the convergence analysis. 

\subsection{Problem Formulation}
Suppose that we are given $n$ i.i.d. samples $\{(x_i,y_i)\}_{i=1}^n$ distributed from a probability distribution $P$ on $\calX \times \Real$ where 
$\calX$ is an input space.
We denote by $\Pi$ the marginal distribution of $P$ on $\calX$. 
We are given $M$ reproducing kernel Hilbert spaces (RKHS) $\{\calH_m\}_{m=1}^M$ each of which is associated with a kernel $k_m$. 
We consider a mixed-norm type regularization with respect to an arbitrary given norm $\|\cdot\|_{\psi}$,
that is, the regularization is given by the norm $\|(\|f_m\hnorm{m})_{m=1}^M\|_{\psi}$ of the vector $(\|f_m\hnorm{m})_{m=1}^M$ for $f_m \in \calH_m$ ($m=1,\dots,M$)\footnote{
We assume that the mixed-norm $\|(\|f_m\hnorm{m})_{m=1}^M\|_{\psi}$ satisfies the triangular inequality with respect to $(f_m)_{m=1}^M$, that is,
$\|(\|f_m + f_m'\hnorm{m})_{m=1}^M\|_{\psi} \leq \|(\|f_m \hnorm{m})_{m=1}^M\|_{\psi} + \|(\|f_m'\hnorm{m})_{m=1}^M\|_{\psi}$. To satisfy this condition,
it is sufficient if the norm is monotone, i.e., $\|\bolda \|_{\psi} \leq \|\bolda + \boldb \|_{\psi}$ for all $\bolda,\boldb \geq \boldsymbol{0}$.}.
For notational simplicity, we write $\|f\|_{\psi} = \|(\|f_m\hnorm{m})_{m=1}^M\|_{\psi}$
for $f=\sum_{m=1}^M f_m~(f_m\in \calH_m)$.

The general formulation of MKL, we consider in this paper, 
fits a function $f = \sum_{m=1}^M f_m~(f_m\in \calH_m)$ to the data by solving the following optimization problem: 
\begin{align}
\fhat = \sum_{m=1}^M \fhat_m = &\mathop{\arg \min}_{f_m \in \calH_m~(m=1,\dots,M)}
\frac{1}{n}\sum_{i=1}^n \left(y_i- \sum_{m=1}^M f_m(x_i) \right)^2 +
\lambdaone \|f \|_{\psi}^2.
\label{eq:primalLp}
\end{align}
We call this ``$\psi$-norm MKL''.
This formulation covers many practically used MKL methods (e.g., $\ell_p$-MKL, elasticnet-MKL, variable sparsity kernel learning (see later for their definitions)),
and is solvable by a finite dimensional optimization procedure due to the representer theorem \citep{JMAA:KimeldorfWahba:1971}.
In this paper, we mainly focus on the regression problem (the squared loss).
However the discussion can be generalized to Lipschitz continuous and strongly convex losses as in \cite{LocalRademacher}
(see Section \ref{sec:GeneralLoss}).

\paragraph{Example 1: $\ell_p$-MKL}
The first motivating example of $\psi$-norm MKL is $\ell_p$-MKL \citep{NIPS:Marius+etal:2009}
that employs $\ell_p$-norm for $1\leq p \leq \infty$ as the regularizer:
$\|f\|_{\psi} = 
\|(\|f_m\hnorm{m})_{m=1}^M\|_{\ell_p} = 
(\sum_{m=1}^M \|f_m\hnorm{m}^p)^{\frac{1}{p}}$.
If $p$ is strictly greater than 1 $(p>1)$, the solution of $\ell_p$-MKL becomes dense.
In particular, $p=2$ corresponds to averaging candidate kernels with uniform weight \citep{JMLR:MicchelliPontil:2005}.
It is reported that $\ell_p$-MKL with $p$ greater than 1, say $p=\frac{4}{3}$, often shows better performance than 
the original sparse $\ell_1$-MKL \citep{ICML:Cortes+etal:gbound:2010}.


\paragraph{Example 2: Elasticnet-MKL}

The second example is elasticnet-MKL \citep{NIPSWS:Taylor:2008,NIPSWS:ElastMKL:2009}
that employs mixture of $\ell_1$ and $\ell_2$ norms as the regularizer:
$\|f\|_{\psi} = \tau \|f\|_{\ell_1} + (1-\tau)\|f\|_{\ell_2} =
\tau \sum_{m=1}^M \|f_m\hnorm{m} + (1-\tau) (\sum_{m=1}^M \|f_m\hnorm{m}^2)^{\frac{1}{2}}$ with $\tau \in [0,1]$.
Elasticnet-MKL shares the same spirit with $\ell_p$-MKL in a sense that it bridges sparse $\ell_1$-regularization and dense $\ell_2$-regularization.
Efficient optimization method for elasticnet-MKL is proposed by \cite{ML:Suzuki+Tomioka:2011}.

\paragraph{Example 3: Variable Sparsity Kernel Learning}
Variable Sparsity Kernel Learning (VSKL) proposed by \cite{JMLR:Aflalo+etal:2011} 
divides the RKHSs into $M'$ groups $\{\calH_{j,k}\}_{k=1}^{M_j},~(j=1,\dots,M')$ 
and 
imposes
a mixed norm regularization $\|f\|_{\psi} = \|f\|_{(p,q)} = \left\{ \sum_{j=1}^{M'} (\sum_{k=1}^{M_j} \|f_{j,k}\hnorm{j,k}^p)^{\frac{q}{p}} \right\}^{\frac{1}{q}}$ 
where $1\leq p,~1\leq q$, and $f_{j,k}\in \calH_{j,k}$.
An advantageous point of VSKL is that by adjusting the parameters $p$ and $q$, 
various levels of sparsity can be introduced. 
The parameters can control the level of sparsity {\it within} group and {\it between} groups.
This point is beneficial especially for multi-modal tasks like object categorization. 

\subsection{Notations and Assumptions}

Here, we prepare notations and assumptions that are used in the analysis.  
Let $\calHtot = \calH_1 \oplus \dots \oplus \calH_M$. 
We utilize the same notation $f \in \calHtot$ indicating both the vector $(f_1,\dots,f_M)$ and the function $f = \sum_{m=1}^M f_m$ ($f_m \in \calH_m$).
This is a little abuse of notation because the decomposition $f = \sum_{m=1}^M f_m$ might not be unique as an element of $\LPi$. 
However this will not cause any confusion.

Throughout the paper, we assume the following technical conditions 
(see also \cite{JMLR:BachConsistency:2008}). 
\begin{Assumption}{\bf(Realizable Assumption)}\\ 
\label{ass:BasicAss}
\noindent {\rm \Assump{\label{ass:truenoise}}}
There exists $\fstar = (\fstar_1,\dots,\fstar_M) \in \calHtot$
such that $\EE[Y|X] = \fstar(X) = \sum_{m=1}^M \fstar_m(X)$,
and the noise $\epsilon := Y - \fstar(X)$ 
is bounded as $|\epsilon| \leq L$.
\end{Assumption} 
\begin{Assumption}{\bf(Kernel Assumption)}\\ 
\label{ass:BasicAss_kernel}
\noindent {\rm \Assump{\label{ass:kernelbound}}}
For each $m=1,\dots,M$, $\calH_m$ is separable (with respect to the RKHS norm) and $\sup_{X\in \calX} |k_m(X,X)| \leq 1$.
\end{Assumption}
The first assumption in (A1) ensures the model $\calHtot$ is correctly specified, 
and the technical assumption $|\epsilon| \leq L$ allows $\epsilon f$ to be Lipschitz continuous with respect to $f$.
The noise boundedness can be relaxed to unbounded situation as in~\cite{arXiv:Raskutti+Martin:2010}
if we consider Gaussian noise, but
we don't pursue that direction for simplicity.
  

Let an integral operator $T_{k_m}:\LPi \to \LPi$ corresponding to a kernel function $k_m$ be 
$$
T_{k_m} f =  \int k_m(\cdot,x) f(x) \dd \Pi(x).
$$
It is known that this operator is compact, positive, and self-adjoint (see Theorem 4.27 of \cite{Book:Steinwart:2008}).
Thus it has at most countably many non-negative eigenvalues.
We denote by $\mu_{\ell,m}$ be the $\ell$-th largest eigenvalue (with possible multiplicity) of the integral operator $T_{k_m}$. 
By Theorem 4.27 of \cite{Book:Steinwart:2008}, the sum of $\mu_{\ell,m}$ is bounded ($\sum_{\ell} \mu_{\ell,m} < \infty$),
and thus $\mu_{\ell,m}$ decreases with order $\ell^{-1}$ ($\mu_{\ell,m} = o(\ell^{-1})$).
We further assume the sequence of the eigenvalues converges even faster to zero.

\begin{Assumption}{\bf (Spectral Assumption)}
\label{eq:specass}
There exist $0 < \sm < 1$ and $0 < c$ such that 
\begin{flalign*}
\text{\rm \Assump}&& 
\mu_{\ell,m} \leq c \ell^{-\frac{1}{\sm}},~~~(\forall \ell \geq 1,~1\leq \forall m \leq M), && 
\end{flalign*}
where $\{\mu_{\ell,m}\}_{\ell=1}^\infty$ is the spectrum of the operator $T_{k_m}$ corresponding to the kernel $k_m$. 
\end{Assumption}
It was shown that the spectral assumption (A3) is equivalent to 
the classical covering number assumption~\citep{COLT:Steinwart+etal:2009}.
Recall that 
the $\epsilon$-covering number $N(\epsilon,\mathcal{B}_{\calH_m},\LPi)$ with respect to $\LPi$
is the minimal number of balls with radius $\epsilon$ needed to cover the unit ball $\mathcal{B}_{\calH_m}$ in $\calH_m$ \citep{Book:VanDerVaart:WeakConvergence}.
If the spectral assumption (A3) and the boundedness assumption (A\ref{ass:kernelbound}) holds, there exists a constant $C$ that
depends only on $s$ and $c$ such that 
\begin{align}
\label{eq:coveringcondition}
\textstyle
\log N(\varepsilon,\mathcal{B}_{\calH_m},\LPi) \leq C \varepsilon^{-2 \sm},
\end{align}
and the converse is also true (see \citet[Theorem 15]{COLT:Steinwart+etal:2009} and \cite{Book:Steinwart:2008} for details).
Therefore, 
if $\sm$ is large, 
the RKHSs are regarded as ``complex'',
and if $\sm$ is small, the RKHSs are ``simple''.

An important class of RKHSs where $\sm$ is known is Sobolev space.
(A\ref{eq:specass}) holds with $\sm=\frac{d}{2\alpha}$ for Sobolev space $W^{\alpha,2}(\calX)$ of $\alpha$-times continuously differentiability on the Euclidean ball $\calX$ of $\Real^d$ 
\citep{Book:Edmunds+Triebel:1996}.
Moreover, for $\alpha$-times differentiable kernels on a closed Euclidean ball in $\Real^d$, 
(A\ref{eq:specass}) holds for $\sm=\frac{d}{2\alpha}$ \citep[Theorem 6.26]{Book:Steinwart:2008}. 
According to Theorem 7.34 of \cite{Book:Steinwart:2008}, 
for Gaussian kernels with compact support distribution, that holds for arbitrary small $0<\sm$.
The covering number of Gaussian kernels with {\it unbounded} support distribution is also described in Theorem 7.34 of \cite{Book:Steinwart:2008}.

Let $\kmin$ be defined as follows:
\begin{align}
\label{eq:defkmin}
\kmin &
\textstyle 
:= \sup\left\{\kappa \geq 0 ~\Big |~ 
\kappa \leq  \frac{\|\sum_{m=1}^M f_m\|_{\LPi}^2}{\sum_{m=1}^M \|f_m\|_{\LPi}^2} ,~\forall f_m \in \calH_m~(m=1,\dots,M)\right\}. 
\end{align}
$\kmin$ represents the correlation of RKHSs. 
We assume all RKHSs are not completely correlated to each other.
\begin{Assumption}{\bf (Incoherence Assumption)}
\label{ass:incoherence}
$\kappa_M$ is strictly bounded from below; there exists a constant $C_0>0$ such that 
\begin{flalign*}
\text{\rm \Assump} 
 &&\textstyle 0 <C_0^{-1} < \kmin. && 
\end{flalign*}
\end{Assumption}
This condition is motivated by the {\it incoherence condition} \citep{COLT:Koltchinskii:2008,AS:Meier+Geer+Buhlmann:2009}
considered in sparse MKL settings. This ensures the uniqueness of the decomposition $\fstar = \sum_{m=1}^M \fstar_m$ of the ground truth.
\cite{JMLR:BachConsistency:2008} also assumed this condition to show the consistency of $\ell_1$-MKL.

Finally we give a technical assumption with respect to $\infty$-norm.
\begin{Assumption}{\bf (Embedded Assumption)}
\label{ass:linfbound}
Under the Spectral Assumption, there exists a constant $C_1>0$ such that 
\begin{flalign*}
\text{\rm \Assump }
&&\textstyle \|f_m\|_{\infty} \leq C_1 \|f_m\hnorm{m}^{1-\sm} \|f_m\|_{\LPi}^{\sm}. &&
\end{flalign*}
\end{Assumption}
This condition is met when the input distribution $\Pi$ has a density with respect
to the uniform distribution on $\calX$ that is bounded away from 0 and
the RKHSs are continuously embedded in a Sobolev space $W^{\alpha,2}(\calX)$ 
where $\sm=\frac{d}{2\alpha}$, 
$d$ is the dimension of the input space $\calX$ and 
$\alpha$ is the ``smoothness'' of the Sobolev space. 
Many practically used kernels satisfy this condition (A5). 
For example, the RKHSs of Gaussian kernels can be embedded in all Sobolev spaces.
Therefore the condition (A5) seems rather common and practical.
More generally, there is a clear characterization of the condition (A5) in terms of {\it real interpolation of spaces}. 
One can find 
detailed and formal discussions of interpolations in \cite{COLT:Steinwart+etal:2009}, and 
Proposition 2.10 of \cite{Book:Bennett+Sharpley:88} gives the necessary and sufficient condition for the assumption (A5). 

\begin{table}[t]
\centering
\caption{Summary of the constants we use in this article.}
\label{tab:constants}
\begin{tabular}{|c|l|}
\hline
$n$ & The number of samples.  \\ \hline
$M$ & The number of candidate kernels.  \\ \hline
$L$ & The bound of the noise (A\ref{ass:BasicAss_kernel}). \\ \hline 
$c$   & The coefficient for Spectral Assumption; see (A\ref{eq:specass}).\\ \hline 
$\sm$ & The decay rate of spectrum; see (A\ref{eq:specass}). \\ \hline 
$\kmin$ & The smallest eigenvalue of the design matrix; see \Eqref{eq:defkmin}. \\ \hline
$C_1$ & The coefficient for Embedded Assumption; see (A\ref{ass:linfbound}). \\ \hline 
\end{tabular}
\end{table}

Constants we use later are summarized in Table~\ref{tab:constants}.

\section{Convergence Rate of $\psi$-norm MKL}
Here we derive the learning rate of $\psi$-norm MKL in the most general setting. 
We suppose that the number of kernels $M$ can increase along with the number of samples $n$.
The motivation of our analysis is summarized as follows:
\begin{itemize}
\item Give a unifying framework to derive a sharp convergence rate of $\psi$-norm MKL.
\item (homogeneous complexity) Show the convergence rate of some examples using our general framework,
prove its minimax-optimality,
and show the optimality of $\ell_1$-regularization
under conditions that the complexities $\sm$ of all RKHSs are same.
\item (inhomogeneous complexity) Discuss how the dense type regularization outperforms sparse type regularization, when the complexities $\sm$ of all RKHSs are {\it not} uniformly same.  
\end{itemize}

We define 
$$
\eta(t) := \eta_{n}(t) = \max(1,\sqrt{t},t/\sqrt{n}),
$$
for $t>0$.
For given positive reals $\{\cm\}_{m=1}^M$ and given $n$, we define $\alpha_1,\alpha_2,\beta_1,\beta_2$ as follows:  
\begin{align}
& \textstyle 
\alpha_1 := \alpha_1(\{\cm\}) = 3 \left( \sum\limits_{m=1}^M \frac{\cm^{-2\sm}}{n}\right)^{\frac{1}{2}},~~
\alpha_2 := \alpha_2(\{\cm\}) = 3  \left\| \left(\frac{\sm \cm^{1-\sm}}{\sqrt{n}}\right)_{m=1}^M \right\|_{\psi^*}, \notag \\
& \textstyle
\beta_1 := \beta_1(\{\cm\}) = \! 3 \left( \sum\limits_{m=1}^M \frac{\cm^{-\frac{2\sm(3-\sm)}{1+\sm}}}{n^{\frac{2}{1+\sm}}} \right)^{\frac{1}{2}}\!\!,~
\beta_2 := \beta_2(\{\cm\}) = \! 3 \left\| \left(\frac{\sm \cm^{\frac{(1-\sm)^2}{1+\sm}}}{n^{\frac{1}{1+\sm}}} \right)_{m=1}^M \right\|_{\psi^*},
\label{eq:defalphabeta}
\end{align}
(note that $\alpha_1,\alpha_2,\beta_1,\beta_2$ implicitly depends on the reals $\{\cm\}_{m=1}^M$).
Then the following theorem gives the general form of the learning rate of $\psi$-norm MKL. 

\begin{Theorem}
\label{th:convergencerateofLpMKL}
Suppose Assumptions 1-5 are satisfied. 
Let $\{\cm\}_{m=1}^M$ be arbitrary positive reals that can depend on $n$, and 
assume  
$\lambdaone \geq \left(\frac{\alpha_2}{\alpha_1}\right)^2 + \left(\frac{\beta_2}{\beta_1}\right)^2$.
Then 
there exists a constant $\phi$
depending only on $\{s_m\}_{m=1}^M$, $c$, $C_1$, $L$
such that 
for all $n$ and $t'$ that satisfy  $\frac{\log(M)}{\sqrt{n}}\leq 1$ and 
$\frac{4\phi \sqrt{n}}{\kminrho} \max\{\alpha_1^2,\beta_1^2,\frac{M\log(M)}{n} \}\eta(t') \leq \frac{1}{12}$
and for all $t \geq 1$,
we have 
\begin{align}
&\|\fhat - \fstar\|_{\LPi}^2 
\leq 
\frac{24 \eta(t)^2 \phi^2}{\kminrho} \left(\alpha_1^2 + \beta_1^2 + \frac{M\log(M)}{n} \right)
+
4 \lambdaone  \|\fstar \|_{\psi}^2,
\label{eq:ConvRateInMainTh_withlambdaone}
\end{align}
with probability $1- \exp(- t) - \exp(-t')$.
In particular, for 
$\lambdaone = \left(\frac{\alpha_2}{\alpha_1}\right)^2 + \left(\frac{\beta_2}{\beta_1}\right)^2$, we have
\begin{align}
&\|\fhat - \fstar\|_{\LPi}^2 
\leq 
\frac{24 \eta(t)^2 \phi^2}{\kminrho} \left(\alpha_1^2 + \beta_1^2 + \frac{M\log(M)}{n} \right)
+
4 \left[ \left(\frac{\alpha_2}{\alpha_1}\right)^2 + \left(\frac{\beta_2}{\beta_1}\right)^2\right]  \|\fstar \|_{\psi}^2.
\label{eq:ConvRateInMainTh}
\end{align}

\end{Theorem}
The proof will be given in Appendix \ref{sec:ProofMainTh}.
The statement of Theorem \ref{th:convergencerateofLpMKL} itself is complicated.
Thus we will show later concrete learning rates on some examples such as $\ell_p$-MKL.
The convergence rate \eqref{eq:ConvRateInMainTh} depends on the positive reals $\{\cm\}_{m=1}^M$,
but the choice of $\{\cm\}_{m=1}^M$ are arbitrary.
Thus by minimizing the right hand side of \Eqref{eq:ConvRateInMainTh}, 
we obtain tight convergence bound as follows:
\begin{align}
&\|\fhat - \fstar\|_{\LPi}^2 
\! = \! \mathcal{O}_p\Bigg(\!
\min_{\substack{ \{\cm\}_{m=1}^M: \\ \cm > 0 }}\! \Bigg\{
\alpha_1^2 + \beta_1^2 + 
\left[ \! \left(\frac{\alpha_2}{\alpha_1}\right)^2 \! + \left(\frac{\beta_2}{\beta_1}\right)^2\right]  \|\fstar \|_{\psi}^2
+
\frac{M\log(M)}{n}\Bigg\} \Bigg).
\label{eq:simpleConveRatemin}
\end{align}
There is a trade-off between the first two terms $(a) := \alpha_1^2 + \beta_1^2$ and the third term
$(b) := \left[ \left(\frac{\alpha_2}{\alpha_1}\right)^2 + \left(\frac{\beta_2}{\beta_1}\right)^2\right]  \|\fstar \|_{\psi}^2$,
that is, if we take $\{\cm\}_m$ large, then the term (a) becomes small and the term (b) becomes large,
on the other hand, if we take $\{\cm\}_m$ small, then it results in large (a) and small (b).
Therefore we need to balance the two terms (a) and (b) to obtain the minimum in \Eqref{eq:simpleConveRatemin}.

We discuss the obtained learning rate in two situations, (i) {\it homogeneous complexity} situation, and (ii) {\it inhomogeneous complexity} situation: \\
~~~~(i) (homogeneous)  All $\sm$s are same: there exists $0<s<1$ such that $\sm = s~(\forall m)$ (Sec.\ref{sec:homogeneous}). \\
~~~~(ii) (inhomogeneous) All $\sm$s are {\it not} same: there exist $m,m'$ such that $\sm \neq \smm{m'}$ (Sec.\ref{sec:inhomogeneous}).

\section{Analysis on Homogeneous Settings}
\label{sec:homogeneous}
Here we assume all $\sm$s are same, say $\sm = s$ for all $m$ (homogeneous setting). 
In this section, we give a simple upper bound of the minimum of the bound \eqref{eq:simpleConveRatemin} (Sec.\ref{sec:simplerateHom}),
derive concrete convergence rates of some examples using the simple upper bound (Sec.\ref{sec:RatesOfExamples}) 
and show that the simple upper bound achieves the minimax learning rate of $\psi$-norm ball if $\psi$-norm is isotropic (Sec.\ref{sec:MinimaxRate}).
Finally we discuss the optimal regularization (Sec.\ref{sec:OptimalReg}).
In Sec.\ref{sec:RatesOfExamples}, we also discuss the difference between our bound of \ellp-MKL and existing bounds.

\subsection{Simplification of Convergence Rate} 
\label{sec:simplerateHom}
If we restrict the situation as all $\cm$s are same ($\cm = \cmm{}~(\forall m)$ for some $\cmm{}$),
then the minimization in \Eqref{eq:simpleConveRatemin} can be easily carried out as in the following lemma.
Let $\boldone$ be the $M$-dimensional vector each element of which is $1$: $\boldone := (1,\dots,1)^\top \in \Real^M$,
and $\|\cdot\|_{\psi^*}$ be the dual norm of the $\psi$-norm\footnote{The dual of the norm $\|\cdot\|_{\psi}$
is defined as $\|\boldb\|_{\psi^*}:= \sup_{\bolda}\{\boldb^\top \bolda \mid \|\bolda\|_{\psi}\leq 1\}$.}.

\begin{Lemma}
\label{lem:smcmuniformbound}

Suppose $\sm = s~(\forall m)$ with some $0<s<1$,   
and set
$\lambdaone 
= 18 M^{\frac{1-\smm{}}{1+\smm{}}} n^{-\frac{1}{1+\smm{}}}\|\boldone\|_{\psi^*}^{\frac{2\smm{}}{1+\smm{}}} \|f^*\|_{\psi}^{-\frac{2}{1+\smm{}}}$,
then 
for all $n$ and $t'$ that satisfy  
$\frac{4\phi}{\kminrho} 
\left\{
9  \left(\frac{M}{\sqrt{n}}\right)^{\frac{1-\smm{}}{1+\smm{}}}
(\|\boldone\|_{\psi^*} \|f^*\|_{\psi})^{\frac{2\smm{}}{1+\smm{}}} \vee \frac{M\log(M)}{\sqrt{n}}
 \right\}
\eta(t') \leq \frac{1}{12}$ and 
$n \geq  (\|\boldone\|_{\psi^*} \|\fstar\|_{\psi}/M)^{\frac{4s}{1-s}}$,
and for all $t \geq 1$,
we have 
\begin{align}
\|\fhat - \fstar\|_{\LPi}^2 
\leq 
C \eta(t)^2  \left\{
M^{1-\frac{2\smm{}}{1+\smm{}}} n^{-\frac{1}{1+\smm{}}}(\|\boldone\|_{\psi^*} \|f^*\|_{\psi})^{\frac{2\smm{}}{1+\smm{}}}
 + \frac{M\log(M)}{n} \right\}, \notag
\end{align}
with probability $1- \exp(- t) - \exp(-t')$ where $C$ is a constant depending on $\phi$ and $\kminrho$.
In particular we have
\begin{align}
\|\fhat - \fstar\|_{\LPi}^2 = \mathcal{O}_p\Bigg\{&
M^{1-\frac{2s}{1+s}} n^{-\frac{1}{1+s}} (\|\boldone\|_{\psi^*} \|\fstar\|_{\psi})^{\frac{2s}{1+s}}
+ \frac{M\log(M)}{n}\Bigg\}.
\label{eq:convergenceRateSimple}
\end{align}
\end{Lemma}
The proof is given in Appendix \ref{sec:smcmuniformboundproof}.
Lemma \ref{lem:smcmuniformbound} is derived by assuming $\cm = \cmm{}~(\forall m)$, which might make the bound loose.
However, when the norm $\|\cdot\|_{\psi}$ is {\it isotropic} (whose definition will appear later), 
that restriction ($\cm = \cmm{}~(\forall m)$) does not make the bound loose,
that is, the upper bound obtained in Lemma \ref{lem:smcmuniformbound} is tight and achieves the minimax optimal rate 
(the minimax optimal rate is the one that cannot be improved by any estimator).
In the following, we investigate the general result of Lemma \ref{lem:smcmuniformbound} through some important examples.

\subsection{Convergence Rate of Some Examples} 
\label{sec:RatesOfExamples}
\subsubsection{Convergence Rate of $\ell_p$-MKL}
Here we derive the convergence rate of $\ell_p$-MKL ($1 \leq p \leq \infty$)
where 
$
\|f\|_{\psi} = \sum_{m=1}^M (\|f_m \hnorm{m}^p)^{\frac{1}{p}}
$
(for $p=\infty$, it is defined as $\max_m  \|f_m \hnorm{m}$).
It is well known that the dual norm of $\ell_p$-norm is given as $\ell_q$-norm where $q$ is the real satisfying $\frac{1}{p} + \frac{1}{q}=1$.
For notational simplicity, let $R_p := \left( \sum_{m=1}^M \|\fstar_m \hnorm{m}^p \right)^{\frac{1}{p}}$.
Then substituting $\|\fstar\|_{\psi} = R_p$ and 
$\|\boldone\|_{\psi^*} = \|\boldone\|_{\ell_q} = M^{\frac{1}{q}} = M^{1 - \frac{1}{p}}$
into the bound \eqref{eq:convergenceRateSimple},
the learning rate of $\ell_p$-MKL is given as 
\begin{align}
\|\fhat - \fstar\|_{\LPi}^2 = &\mathcal{O}_p\Big( n^{-\frac{1}{1+s}} M^{1-\frac{2s}{p(1+s)}} R_p^{\frac{2s}{1+s}} + \frac{M\log(M)}{n} 
\Big). 
\label{eq:convergenceRateComplex}
\end{align}
If we further assume $n$ is sufficiently large such that 
\begin{equation}
\label{eq:nboundInOurs}
n\geq M^{\frac{2}{p}} R_p^{-2} (\log M)^{\frac{1+s}{s}},
\end{equation} 
then the leading term is the first term,
and thus we have 
\begin{equation}
\label{eq:convergenceRateSimple_Simple}
\textstyle
\|\fhat - \fstar\|_{\LPi}^2 = \mathcal{O}_p \left(n^{-\frac{1}{1+s}} M^{1-\frac{2s}{p(1+s)}} R_p^{\frac{2s}{1+s}}\right).
\end{equation}
Note that as the complexity $s$ of RKHSs becomes small the convergence rate becomes fast.  
It is known that $n^{-\frac{1}{1+s}}$ is the minimax optimal learning rate for single kernel learning.
The derived rate of \ellp-MKL is obtained by multiplying a coefficient depending on $M$ and $R_p$ to the optimal rate of single kernel learning.
To investigate the dependency of $R_p$ to the learning rate, let us consider two extreme settings, i.e., 
sparse setting $(\|\fstar_m\hnorm{m})_{m=1}^M = (1,0,\dots,0)$ and dense setting $(\|\fstar_m\hnorm{m})_{m=1}^M = (1,\dots,1)$
as in \cite{JMLR:lp:Marius+etal:2011}.
\begin{itemize}
\item $(\|\fstar_m\hnorm{m})_{m=1}^M = (1,0,\dots,0)$: $R_p=1$ for all $p$. 
Therefore the convergence rate $n^{-\frac{1}{1+s}} M^{1-\frac{2s}{p(1+s)}}$ is fast for small $p$ and the minimum is achieved at $p=1$.
This means that $\ell_1$ regularization is preferred for sparse truth.  
\item $(\|\fstar_m\hnorm{m})_{m=1}^M = (1,\dots,1)$: $R_p = M^{\frac{1}{p}}$, thus the convergence rate is $M n^{-\frac{1}{1+s}}$ for all $p$. 
Interestingly for dense ground truth, there is no dependency of the convergence rate on the parameter $p$
(later we will show that this is not the case in inhomogeneous setting (Sec.\ref{sec:inhomogeneous})). 
That is, the convergence rate is $M$ times the optimal learning rate of single kernel learning ($n^{-\frac{1}{1+s}}$) for all $p$.
This means that for the dense settings, the complexity of solving MKL problem is equivalent to that of solving $M$ single kernel learning problems.
\end{itemize}


\paragraph{Comparison with Existing Bounds}

Here we compare the bound for $\ell_p$-MKL we derived above with the existing bounds. 
Let $\calHlp(R_p)$ be the $\ell_p$-mixed norm ball with radius $R_p$:
$
\textstyle \calHlp(R_p)  := \{f = \sum_{m=1}^M f_m \mid ( \sum_{m=1}^M \|f_m\hnorm{m}^p )^{\frac{1}{p}} \leq R_p \}.
$
There are two types of convergence rates: global bound and localized bound.

~

\noindent 
{\bf (comparison with existing global bound)}
\cite{ICML:Cortes+etal:gbound:2010,ECML:Marius+etal:2010,JMLR:lp:Marius+etal:2011} gave ``global'' type bounds for $\ell_p$-MKL as 
\begin{align}
\textstyle R(f) \leq \widehat{R}(f) + C 
\begin{cases} \sqrt{\frac{\log(M)}{n}}R_p  &(p=1),\\ 
\frac{M^{1-\frac{1}{p}}}{\sqrt{n}}R_p& (p>1), \end{cases}~~~
(\text{for all $f \in \calH_{\ell_p}(R_p)$}),  
\label{eq:CortesBound}
\end{align}
where $R(f)$ and $\widehat{R}(f)$ is the population risk and the empirical risk.
The bounds by \cite{ICML:Cortes+etal:gbound:2010} and \cite{JMLR:lp:Marius+etal:2011} are restricted to the situation $1\leq p \leq 2$.
On the other hand, our analysis and that of \cite{ECML:Marius+etal:2010} covers all $p \geq 1$.

Since our bound is specialized to the regularized risk minimizer $\hat{f}$ defined at \Eqref{eq:primalLp}
while the existing bound \eqref{eq:CortesBound} is applicable to all $f \in \calH_{\ell_p}(R_p)$,
our bound is sharper than theirs for sufficiently large $n$. 
To see this, suppose that 
\begin{equation}
\label{eq:nboundInGlobalVSLocal}
n \geq
\begin{cases}  M^{2} R_1^{-2} (\log M)^{-\frac{1+s}{1-s}}  &(p=1),\\ 
 M^{\frac{2}{p}} R_p^{-2} & (p>1), \end{cases}~~~
\end{equation}
then we have $n^{-\frac{1}{1+s}} M^{1-\frac{2s}{p(1+s)}} R_p^{\frac{2s}{1+s}} \leq n^{-\frac{1}{2}} (M^{1-\frac{1}{p}} \vee \log(M))  R_p$ 
and hence our localized bound is sharper than the global one.
Interestingly, the range of $n$ presented in \Eqref{eq:nboundInGlobalVSLocal} where the localized bound exceeds the global bound 
is same (up to $\log M$ term) 
as the range presented in \Eqref{eq:nboundInOurs} ($n\geq M^{\frac{2}{p}} R_p^{-2} (\log M)^{\frac{1+s}{s}}$)
where the first term in our bound \eqref{eq:convergenceRateComplex} dominates its second term
so that the simplified bound \eqref{eq:convergenceRateSimple_Simple} holds. 
That means that, at the ``phase transition point'' from global to localized bound, 
the first informative term in our bound becomes the leading term. 


Finally we note that, since $s$ can be large as long as Spectral Assumption (A3) is satisfied,
the bound \eqref{eq:CortesBound} is recovered by our analysis by approaching $s$ to 1.

~

\noindent
{\bf (comparison with existing localized bound)}
Recently \cite{arXiv:Kloft+Gilles:2011} gave a tighter convergence rate utilizing the localization technique as 
\begin{equation}
\textstyle \|\fhat - \fstar\|_{\LPi}^2 = \mathcal{O}_p \Big(\min_{p'\geq p} \Big\{ 
\frac{p'}{p'-1}
n^{-\frac{1}{1+s}} M^{1-\frac{2s}{p'(1+s)}} R_{p'}^{\frac{2s}{1+s}}
\Big\}\Big),
\label{eq:MariusLpBound}
\end{equation}
under a strong condition $\kminrho=1$ that imposes all RKHSs are completely uncorrelated to each other. 
Comparing our bound with their result, there is $\min_{p'\geq p}$ and $\frac{p'}{p'-1}$ in their bound
(if there is not the term  $\frac{p'}{p'-1}$, then the minimum of $\min_{p'\geq p}$ is attained at $p'=p$, thus our bound is tighter).
Due to this, we obtain a quite different consequence from theirs.
According to our bound \eqref{eq:convergenceRateSimple_Simple}, the optimal regularization among all \ellp-norm 
that gives the smallest generalization error is $\ell_1$-regularization (this will be discussed later in Sec.\ref{sec:OptimalReg})
while their consequence says that the optimal $p$ changes depending on the ``sparsity'' of the true function $\fstar$.
Moreover we will observe that $\ell_1$-regularization is optimal among \emph{all} isotropic mixed-norm-type regularization. 
The details of the optimality will be discussed in Sec.\ref{sec:OptimalReg}.


\subsubsection{Convergence Rate of Elasticnet-MKL}
Elasticnet-MKL employs a mixture of $\ell_1$ and $\ell_2$ norm as the regularizer:
$$
\|f\|_{\psi} = \tau \|f\|_{\ell_1} + (1-\tau) \|f\|_{\ell_2}$$ 
where $\tau \in [0,1]$.

Then its dual norm is given by
$
\| \boldb \|_{\psi*} = \min_{\bolda \in \Real^M}\left\{ \max\left(\frac{\|\bolda\|_{\ell_{\infty}}}{\tau},\frac{\|\bolda - \boldb\|_{\ell_2}}{1-\tau}\right) \right\}.
$
Therefore by a simple calculation, we have 
$
\| \boldone \|_{\psi*} = \frac{\sqrt{M}}{1-\tau + \tau \sqrt{M}}.
$
Hence \Eqref{eq:convergenceRateSimple} gives the convergence rate of elasticnet-MKL as 
\begin{align*}
\|\fhat - \fstar\|_{\LPi}^2 = \mathcal{O}_p\Bigg(
n^{-\frac{1}{1+s}} \frac{M^{1-\frac{s}{1+s}}}{(1-\tau + \tau \sqrt{M})^{\frac{2s}{1+s}}}  (\tau \|\fstar\|_{\ell_1} + (1-\tau) \|\fstar\|_{\ell_2})^{\frac{2s}{1+s}}
+ \frac{M\log(M)}{n}\Bigg).
\end{align*}
Note that, when $\tau= 0$ or $\tau= 1$, this rate is identical to that of $\ell_2$-MKL or $\ell_1$-MKL obtained in \Eqref{eq:convergenceRateComplex} respectively.

\subsubsection{Convergence Rate of VSKL}
Variable Sparsity Kernel Learning (VSKL) employs a mixed norm regularization  
defined by 
$$
\textstyle \|f\|_{\psi} = \|f\|_{(p,q)} = \left\{ \sum_{j=1}^{M'} \left(\sum_{k=1}^{M_j} \|f_{j,k}\hnorm{j,k}^p\right)^{\frac{q}{p}} \right\}^{\frac{1}{q}},
$$
where RKHSs are divided into $M'$ groups $\{\calH_{j,k}\}_{k=1}^{M_j},~(j=1,\dots,M')$
and $1\leq p,~1\leq q$.

\begin{Lemma}
\label{lemm:dualmixednorm}
The dual of the mixed norm is given by 
$$
\textstyle  
\|\boldb\|_{\psi^*} = \left\{ \sum_{j=1}^{M'} \left(\sum_{k=1}^{M_j} |b_{j,k}|^{p^*}\right)^{\frac{q^*}{p^*}} \right\}^{\frac{1}{q^*}},
$$
for $b_{j,k} \in \Real~(k=1,\dots,M_j,~j=1,\dots,M')$.
\end{Lemma}
The proof will be given in Appendix \ref{sec:dualofnorms}.
Therefore the dual norm of the vector $\boldone$ is given by 
$\|\boldone\|_{\psi^*} = \left(\sum_{j=1}^{M'} M_j^{\frac{q^*}{p^*}}\right)^{\frac{1}{q^*}}$.
Hence, by \Eqref{eq:convergenceRateSimple}, the convergence rate of VSKL is given as 
\begin{align*}
&\|\fhat - \fstar\|_{\LPi}^2  \\
= 
&\mathcal{O}_p\Bigg(
n^{-\frac{1}{1+s}} \left(\sum_{j=1}^{M'}M_j\right)^{1-\frac{2s}{1+s}} 
\left[\left(\sum_{j=1}^{M'} M_j^{\frac{q^*}{p^*}}\right)^{\frac{1}{q^*}} \left\{ \sum_{j=1}^{M'} (\sum_{k=1}^{M_j} \|f^*_{j,k}\hnorm{j,k}^p)^{\frac{q}{p}} \right\}^{\frac{1}{q}} \right]^{\frac{2s}{1+s}}
+ \frac{M\log(M)}{n}\Bigg).
\end{align*}
One can check that this convergence rate coincides with that of \ellp-MKL when $M'=1$. 

\subsection{Minimax Lower Bound} 
\label{sec:MinimaxRate}
In this section, we show that the derived learning rate \eqref{eq:convergenceRateSimple}
achieves the minimax-learning rate on 
the $\psi$-norm ball
$$
\textstyle \calHpsi(R)  := \left\{f = \sum_{m=1}^M f_m \bmid \| f\|_{\psi} \leq R \right\},
$$
when the norm is {\it isotropic}.
\begin{definition}
We say that $\psi$-norm $\|\cdot\|_{\psi}$ is isotropic when there exits a universal constant $\bar{c}$ such that 
\begin{equation}
\bar{c} M = \bar{c} \|\boldone\|_{\ell_1} \geq \|\boldone\|_{\psi^*}\|\boldone\|_{\psi},
~~~~~~~~\|\boldb \|_{\psi} \leq \|\boldb'\|_{\psi}~~(\text{if}~0 \leq b_m \leq b_m'~(\forall m)),
\label{eq:defistropic}
\end{equation}
(note that the inverse inequality $M\leq \|\boldone\|_{\psi^*}\|\boldone\|_{\psi}$ of the first condition always holds by the definition of the dual norm).
\end{definition}
Practically used regularizations usually satisfy the isotropic property. 
In fact, $\ell_p$-MKL, elasticnet-MKL and VSKL satisfy the isotropic property with $\bar{c} = 1$.

We derive the minimax learning rate in a simpler situation.
First we assume that each RKHS is same as others.
That is, the input vector is decomposed into $M$ components like $x=(x^{(1)},\dots,x^{(M)})$ where $\{x^{(m)}\}_{m=1}^M$ are $M$ i.i.d. copies of a random variable $\tilde{X}$,
and $\calH_m = \{f_m \mid f_m(x) = f_m(x^{(1)},\dots,x^{(M)}) = \tilde{f}_m(x^{(m)}),~\tilde{f}_m \in \repH \}$ where $\repH$ is an RKHS shared by all $\calH_m$.
Thus $f\in \calHtot$ is decomposed as  
$f(x) = f(x^{(1)},\dots,x^{(M)}) = \sum_{m=1}^M \tilde{f}_m(x^{(m)})$ where each $\tilde{f}_m$ is a member of the common RKHS $\repH$. 
We denote by $\repk$ the kernel associated with the RKHS $\repH$.

In addition to the condition about the upper bound of spectrum (Spectral Assumption (A3)), 
we assume that the spectrum of all the RKHSs $\calH_m$ have the same lower bound of polynomial rate. 
\begin{Assumption}{\bf (Strong Spectral Assumption)}
\label{eq:specass2}
There exist $0 < s < 1$ and $0<c,c'$ such that 
\begin{flalign*}
\text{\rm \Assump{}} &&
c' \ell^{-\frac{1}{s}} \leq \tilde{\mu}_{\ell} \leq c \ell^{-\frac{1}{s}},~~~(1\leq \forall \ell),&&
\end{flalign*}
where $\{\tilde{\mu}_{\ell}\}_{\ell=1}^\infty$ is the spectrum of the integral operator $T_{\tilde{k}}$ corresponding to the kernel $\tilde{k}$. 
In particular, the spectrum of $T_{k_m}$ also satisfies $\mu_{\ell,m} \sim \ell^{-\frac{1}{s}}~(\forall \ell,m)$.
\end{Assumption}
Without loss of generality, we may assume that 
$
\EE[f(\tilde{X})] = 0~(\forall f \in \repH).
$
Since each $f_m$ receives i.i.d. copy of $\tilde{X}$, $\calH_m$s are orthogonal to each other:
\begin{align*}
&\EE[f_m(X) f_{m'}(X)] = \EE[\tilde{f}_m(X^{(m)}) \tilde{f}_{m'}(X^{(m')})] =0 \\
&~~~(\forall f_m \in \calH_m,~\forall f_{m'} \in \calH_{m'},~1\leq \forall m \neq m' \leq M).
\end{align*}
We also assume that the noise $\{\epsilon_i\}_{i=1}^n$ is an i.i.d. normal sequence with standard deviation $\sigma>0$.  

Under the assumptions described above, we have the following minimax $\LPi$-error.
\begin{Theorem}
\label{th:LowerBounds}
Suppose $R>0$ is given and $n > \frac{\bar{c}^2  M^2}{R^2 \|\boldone\|_{\psi^*}^2}$ is satisfied.
Then the minimax-learning rate on $\calHpsi(R)$ for isotropic norm $\|\cdot\|_{\psi}$ is lower bounded as 
\begin{align}
\min_{\hat{f}} \max_{f^* \in \calHpsi(R)} \EE\left[ \|\hat{f} - f^* \|_{\LPi}^2\right] \geq 
C M^{1-\frac{2s}{1+s}} n^{-\frac{1}{1+s}} (\|\boldone\|_{\psi^*} R)^{\frac{2s}{1+s}},
\label{eq:PsinormBallMiniMaxRate}
\end{align}
where $\inf$ is taken over all measurable functions of $n$ samples $\{(x_i,y_i)\}_{i=1}^n$.
\end{Theorem}
The proof will be given in Appendix \ref{sec:proofOfMinimax}.
One can see that the convergence rate derived in \Eqref{eq:convergenceRateSimple}
achieves the minimax rate on the $\psi$-norm ball (Theorem \ref{th:LowerBounds})
up to $\frac{M\log(M)}{n}$ that is negligible when the number of samples is large.
Indeed if 
\begin{equation}
n \geq \frac{M^2\log(M)^{\frac{1+s}{s}}}{\|\boldone\|_{\psi^*}^2\|\fstar\|_{\psi}^2},
\label{eq:nboundGeneralHomogeneuos}
\end{equation} 
then the first term in \Eqref{eq:convergenceRateSimple} dominates the second term $\frac{M\log(M)}{n}$
and the upper bound coincides with the minimax optima rate. 
Note that the condition \eqref{eq:nboundGeneralHomogeneuos} for the sample size $n$
is equivalent to the condition for $n$ assumed in Theorem \ref{th:LowerBounds} up to factors of $\log(M)^{\frac{1+s}{s}}$ and a constant.

The fact that $\psi$-norm MKL achieves the minimax optimal rate \eqref{eq:PsinormBallMiniMaxRate} 
indicates that the $\psi$-norm regularization is well suited to make the estimator included in the $\psi$-norm ball. 

\subsection{Optimal Regularization Strategy} 
\label{sec:OptimalReg}
Here we discuss which regularization gives the best performance
based on the generalization error bound given by Lemma \ref{lem:smcmuniformbound}.
Surprisingly the best regularization that gives the optimal performance among \emph{all isotropic $\psi$-norm} regularizations
is $\ell_1$-norm regularization.
This can be seen as follows.
According to \Eqref{eq:convergenceRateSimple}, we have seen that the convergence rate of $\psi$-norm MKL
is upper bounded as 
\begin{align*}
\|\fhat - \fstar\|_{\LPi}^2 = \mathcal{O}_p\Bigg\{&
M^{1-\frac{2s}{1+s}} n^{-\frac{1}{1+s}} (\|\boldone\|_{\psi^*} \|\fstar\|_{\psi})^{\frac{2s}{1+s}}
+ \frac{M\log(M)}{n}\Bigg\},
\end{align*}
and this is mini-max optimal on $\psi$-norm ball if $\psi$-norm is isotropic.
Here by the definition of the dual norm $\|\cdot\|_{\psi^*}$, we always have 
\begin{equation}
\|\fstar\|_{\ell_1} = \sum_{m=1}^M \|\fstar_m\hnorm{m} = \sum_{m=1}^M 1\times \|\fstar_m \hnorm{m} \leq \|\boldone\|_{\psi^*} \|\fstar\|_{\psi}.
\end{equation}
Therefore the leading term of the convergence rate for $\ell_1$-norm regularization is 
upper bounded by that for other arbitrary $\psi$-norm regularization as 
$$
M^{1-\frac{2s}{1+s}} n^{-\frac{1}{1+s}} \|\fstar\|_{\ell_1}^{\frac{2s}{1+s}}
\leq
M^{1-\frac{2s}{1+s}} n^{-\frac{1}{1+s}} (\|\boldone\|_{\psi^*} \|\fstar\|_{\psi})^{\frac{2s}{1+s}},
$$
(here it should be noticed that the dual norm of $\ell_1$-norm is $\ell_{\infty}$-norm and $\|\boldone\|_{\ell_{\infty}}=1$).
This shows that the upper bound \eqref{eq:convergenceRateSimple} is minimized by $\ell_1$-norm regularization.
In other words, $\ell_1$-regularization is optimal among \emph{all} (isotropic) $\psi$-norm regularization in homogeneous settings.

This consequence is different from that of \cite{arXiv:Kloft+Gilles:2011} where
the optimal regularization among $\ell_p$-MKL is discussed. 
Their consequence says that the best performance is achieved at $p \gneq 1$ and 
the best $p$ depends on the variation of the RKHS norms of $\{\fstar_m\}_{m=1}^M$:
if $\fstar$ is close to sparse (i.e., $\|\fstar_m\hnorm{m}$ decays rapidly), small $p$ is preferred, on the other hand
if $\fstar$ is dense (i.e., $\{\|\fstar_m\hnorm{m}\}_{m=1}^M$ is uniform), then large $p$ is preferred.
This consequence seems reasonable, but
our consequence is different: $\ell_1$-norm regularization is always optimal in $\ell_p$-regularizations.
The antinomy of the two consequences comes from the additional terms $\min_{p'\geq p}$ and $\frac{p'}{p'-1}$
in their bound \eqref{eq:MariusLpBound} (there are no such terms in our bound). 
This difference makes our bound tighter than their bound but simultaneously 
leads to a somewhat counter-intuitive consequence
that is contrastive against the some experiment results supporting dense type regularization. 
However such experimental observations are justified by 
considering \emph{inhomogeneous settings}.
Here we should notice that the homogeneous setting is quite restrictive and unrealistic because
it is required that the complexities of all RKHSs are uniformly same. 
In real settings, it is natural to assume the complexities varies depending on RKHS (inhomogeneous).
In the next section, we discuss how dense type regularizations 
outperform the $\ell_1$-regularization. 


\section{Analysis on Inhomogeneous Settings}
\label{sec:inhomogeneous}
In the previous sections (analysis on homogeneous settings), we have seen $\ell_1$-MKL shows the best performance 
among isotropic $\psi$-norm
and 
have not observed any theoretical justification supporting the fact that dense MKL methods like $\ell_{\frac{4}{3}}$-MKL can 
outperform the sparse $\ell_1$-MKL \citep{ICML:Cortes+etal:gbound:2010}.
In this section, we show 
dense type regularizations can outperform the sparse regularization in inhomogeneous settings 
(where there exists $m,m'$ such that $\smm{m} \neq \smm{m'}$).
For simplicity, we focus on $\ell_p$-MKL, and discuss the relation between the learning rate and the norm parameter $p$.

Let us consider an extreme situation where $\smm{1} = s$ for some $0<s<1$ and $\sm = 0~(m>1)$\footnote{
In our assumption $\sm$ should be greater than 0. However we formally put $\sm = 0$ ($m>1$) for simplicity of discussion.
For rigorous discussion, one might consider arbitrary small $\sm \ll s$.}.
In this situation, we have 
$$
\alpha_1 = 3 \left(\frac{r_1^{-2s} + M-1}{n}\right)^{\frac{1}{2}},\alpha_2 = 3 \frac{s r_1^{1-s}}{\sqrt{n}},
\beta_1 = 3 \bigg(\frac{r_1^{-\frac{2s(3-s)}{1+s}} + M-1}{n^{\frac{2}{1+s}}}\bigg)^{\frac{1}{2}},
\beta_2 = 3 \frac{s r_1^{\frac{(1-s)^2}{1+s}}}{n^{\frac{1}{1+s}}}.
$$
for all $p$. Note that these $\alpha_1$, $\alpha_2$, $\beta_1$ and $\beta_2$ have no dependency on $p$.
Therefore the learning bound 
\eqref{eq:simpleConveRatemin} is smallest when $p=\infty$ 
because $\|\fstar \|_{\ell_{\infty}} \leq \|\fstar \|_{\ell_{p}}$ for all $1\leq p < \infty$.
In particular, when $(\|\fstar_m\hnorm{m})_{m=1}^M = \boldone$, 
we have $\|\fstar\|_{\ell_1} = M\|\fstar\|_{\ell_{\infty}}$ and thus 
obviously the learning rate of $\ell_\infty$-MKL given by \Eqref{eq:simpleConveRatemin} is faster than that of $\ell_1$-MKL.
In fact, through a bit cumbersome calculation, one can check that $\ell_\infty$-MKL can be 
at least $M^{\frac{2s}{1+s}}$ times
faster (up to constants) than $\ell_1$-MKL in a worst case.
Indeed we have the following learning rate of $\ell_1$-MKL and $\ell_\infty$-MKL (say $\fhat^{(1)}$ and $\fhat^{(\infty)}$).
\begin{Lemma}
\label{lem:inhomogeneousConvRate}
Suppose $\smm{1} = s$ for $0<s<1$ and $\smm{m}=0~(m>1)$ and $\|\fstar_m\hnorm{m}=1~(\forall m)$.  
If $n \geq M^{\frac{4s}{1-s}} \vee (M\log(M))^{\frac{1+s}{s}}$, then   
the bound \eqref{eq:simpleConveRatemin} implies 
\begin{align*}
&
\|\fhat^{(1)} - \fstar\|_{\LPi}^2
=\mathcal{O}_p\left( 
n^{-\frac{1}{1+s}} M^{\frac{2\smm{}}{1+\smm{}}}
\right), \\
&
\|\fhat^{(\infty)} - \fstar\|_{\LPi}^2
=\mathcal{O}_p\left(
n^{-\frac{1}{1+s}} 
 \right).
\end{align*}

\end{Lemma}
This indicates that when the complexities of RKHSs are inhomogeneous, the generalization ability of {\it dense} type regularization (e.g., $\ell_\infty$-MKL)
can be better than {\it sparse} type regularization ($\ell_1$-MKL).

Next we numerically calculate the convergence rate: 
\begin{align}
\min_{\substack{ \{\cm\}_{m=1}^M: \\ \cm > 0 }}\! \Bigg\{
\alpha_1^2 + \beta_1^2 + 
\left[ \! \left(\frac{\alpha_2}{\alpha_1}\right)^2 \! + \left(\frac{\beta_2}{\beta_1}\right)^2\right]  \|\fstar \|_{\psi}^2
\Bigg\}. 
\label{eq:convrateLeadingTermWithmin}
\end{align}
Here we randomly generated $\sm$ from the uniform distribution on $[0,1/3]$ and 
$\|\fstar_m\hnorm{m}$ from the uniform distribution on $[0,1]$ with $n=100$ and $M=10$.
Then calculated the minimum of \Eqref{eq:convrateLeadingTermWithmin} using a numerical optimization solver
where $\ell_p$-norm is employed as the regularizer (\ellp-MKL).
We used {\it Differential Evolution} technique\footnote{
We used the Matlab$^{\scriptsize \textcircled{\tiny{R}}}$ code available in \cite{Book:DE-Chakraborty:2008}.} \citep{Book:DE-Price+etal:2005,Book:DE-Chakraborty:2008}
to obtain the minimum value.
Figure \ref{fig:genlpSim} plots the minimum value of \Eqref{eq:convrateLeadingTermWithmin} against 
the parameter $p$ of $\ell_p$-norm. 
We can see that the generalization error once goes down and then goes up as $p$ gets large.
The optimal $p$ is attained around $p=1.4$ in this example.

\begin{figure}
\begin{center}
\includegraphics[width=9cm,clip]{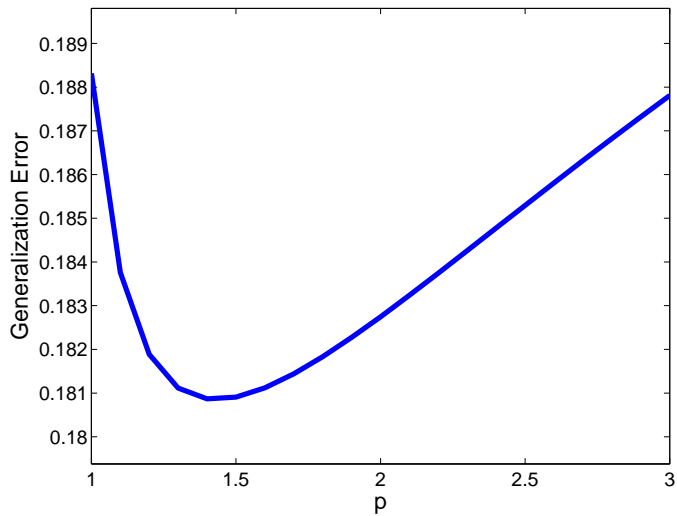}
\caption{The generalization error bound \eqref{eq:convrateLeadingTermWithmin} of \ellp-MKL with respect to $p$.}
\label{fig:genlpSim}
\end{center}
\end{figure}

In real settings, it is likely that one uses various types of kernels and the complexities of RKHSs become inhomogeneous. 
As mentioned above, it has been often reported that $\ell_1$-MKL is outperformed by dense type MKL such as $\ell_{\frac{4}{3}}$-MKL 
in numerical experiments \citep{ICML:Cortes+etal:gbound:2010}.
Our theoretical analysis in this section well support these experimental results. 

\section{Numerical Comparison between Homogeneous and Inhomogeneous Settings}
Here we investigate numerically how the inhomogeneity of the complexities affects the performances
using synthetic data.
In particular, we numerically compare two situations: (a) all complexities of RKHSs are same (homogeneous situation)
and (b) one RKHS is complex and other RKHSs are evenly simple (inhomogeneous situation).

The experimental settings are as follows. 
The input random variable is 20 dimensional vector $x=(x^{(1)},\dots,x^{(20)})$ 
where each element $x^{(m)}$ is independently identically distributed from the uniform distribution on $[0,1]$:
$$
x^{(m)} \sim \mathrm{Unif}([0,1])~~~(m=1,\dots,20).
$$
For each coordinate $m = 1,\dots,20$, we put one Gaussian RKHS $\calH_m$ with a Gaussian width $\sigma_m$:
the number of kernels is 20 ($M=20$) and 
$$
k_m(x,x') = \exp\left(-\frac{(x^{(m)} - {x'^{(m)}})^2}{2\sigma_m^2}\right)~~~(m=1,\dots,20),
$$
for $x=(x^{(1)},\dots,x^{(20)})$ and $x'=(x'^{(1)},\dots,x'^{(20)})$.
To generate the ground truth $\fstar$, 
we randomly generated 5 center points $\mu_{i,m}~(i=1,\dots,5)$ for each coordinate $m = 1,\dots,20$ where
$\mu_{i,m}$ is independently generated by the uniform distribution on $[0,1]$.
Then we obtain the following form of the true function:
\begin{align*}
\fstar(x) &= \sum_{m=1}^{20} \fstar_m(x), \\ 
\text{where}~~\fstar_m(x) &= \sum_{i=1}^5 \alpha_{i,m} \exp\left(-\frac{(x^{(m)} - \mu_{i,m})^2}{2\sigma_m^2}\right) 
\in \calH_m,
\end{align*}
for $x=(x_1,\dots,x_m)$.
Each coefficient $\alpha_{i,m}$ is independently identically distributed from the standard normal distribution. 
The output $y$ is contaminated by a noise $\epsilon$
where the noise $\epsilon$ is distributed from the Gaussian distribution 
with mean 0 and standard deviation 0.1:
\begin{align*}
&y = \fstar_m(x) + \epsilon,  \\
&\epsilon \sim \calN(0,0.1).
\end{align*}

We generated 200 or 400 realizations $\{(x_i,y_i)\}_{i=1}^{n}$ ($n=200$ or $n=400$), 
and estimated $\fstar$ using \ellp-MKL with $p = 1,1.1,1.2,\dots,3$
\footnote{We included a bias term in this experiment, that is, we fitted $\fhat(x) + b$ to the data:
$\min_{f_m,b}\frac{1}{n}\sum_{i=1}^n (y_i - \sum_{m=1}^M f_m(x_i) - b)^2 + \lambdaone \|f\|_{\ell_p}^2$.}.
The estimator is computed with various regularization parameters $\lambdaone$.
The generalization error $\|\fhat - \fstar\|_{\LPi}^2$ was numerically calculated.
We repeated the experiments for 100 times, 
averaged the generalization errors over 100 repetitions for each $p$ and each regularization parameter,
and obtained the optimal average generalization error among all regularization parameters for each $p$.
The true function was randomly generated for each repetition.
We investigated the generalization errors in the following homogeneous and inhomogeneous settings:
\begin{enumerate}
\item (homogeneous) $\sigma_m = 0.5$ for $m=1,\dots,20$.
\item (inhomogeneous) $\sigma_1 = 0.01$ and $\sigma_m = 0.5$ for $m=2,\dots,20$. 
\end{enumerate}
The difference between the above homogeneous and inhomogeneous settings is the value of $\sigma_1$;
whether $\sigma_1=0.5$ or $\sigma_1 = 0.01$.
The inhomogeneous situation is analogous to that investigated in Sec.\ref{sec:inhomogeneous}
where we assumed one RKHS is complex and the other RKHSs are evenly simple 
(small $\sigma_1$ corresponds to a complex RKHS).

\begin{figure}
\begin{center}
\subfigure[Homogeneous Setting ($n=200$)]{
\includegraphics[width=7cm,clip]{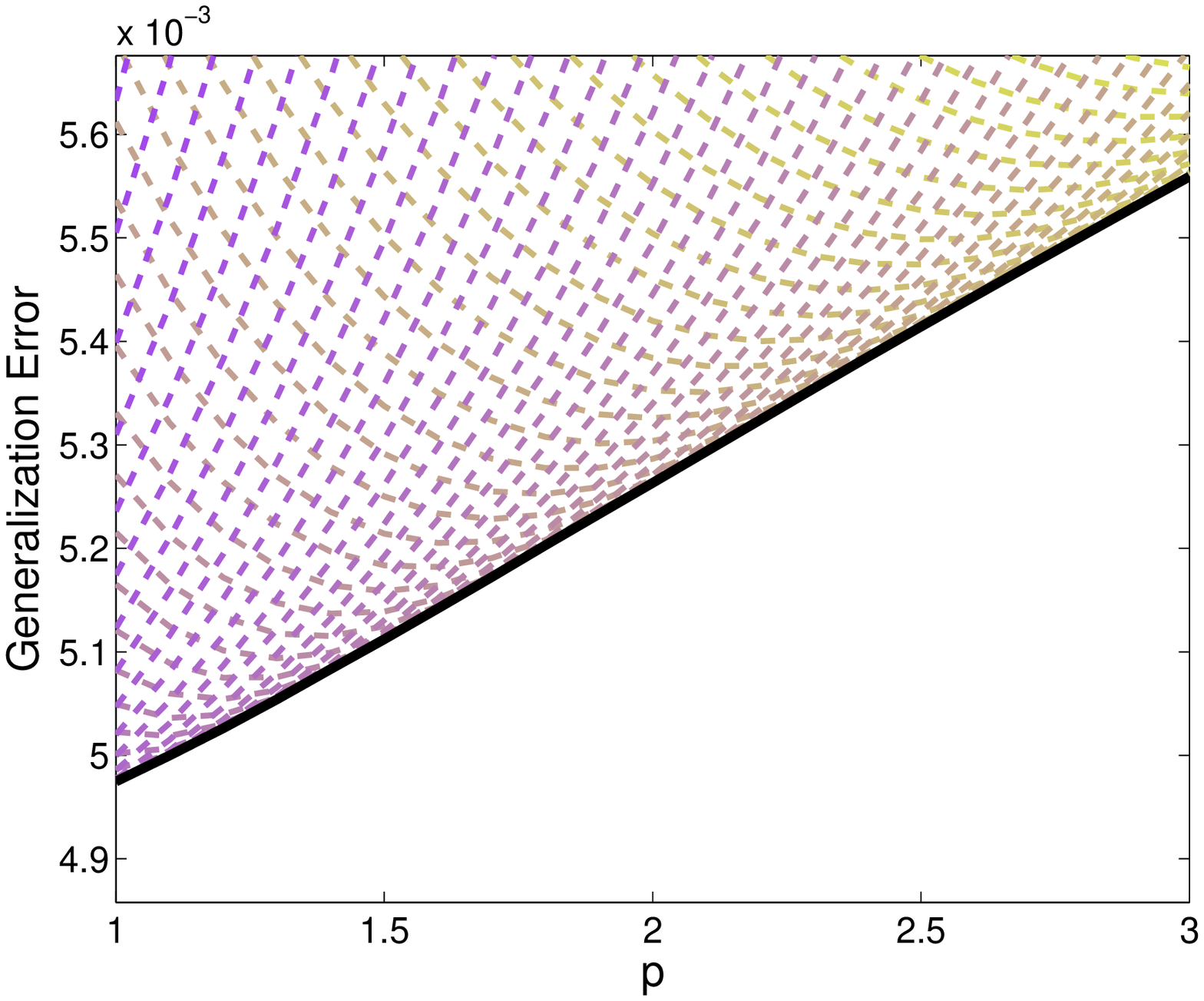}
}
\subfigure[Inhomogeneous Setting ($n=200$)]{
\includegraphics[width=7cm,clip]{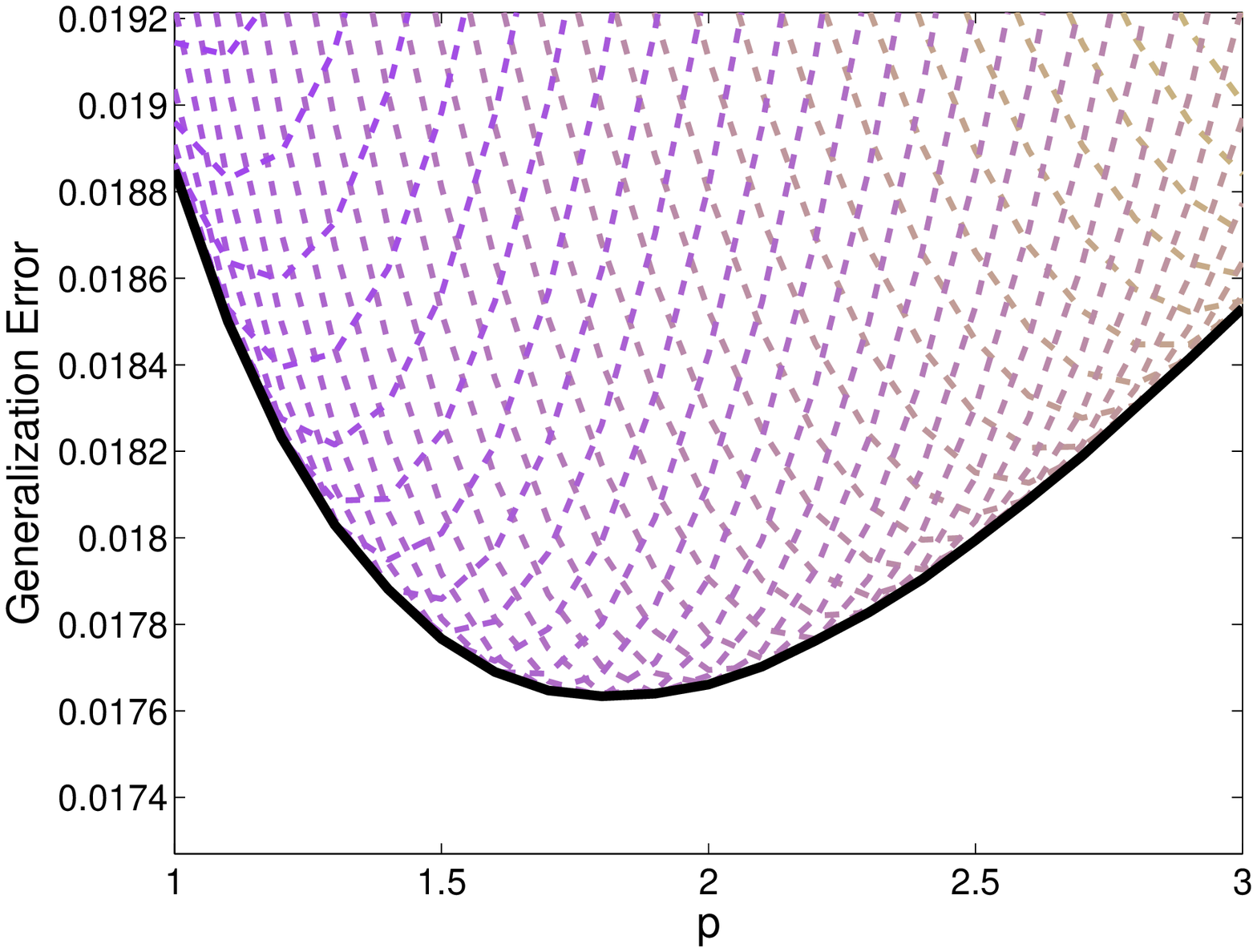}
}
\subfigure[Homogeneous Setting ($n=400$)]{
\includegraphics[width=7cm,clip]{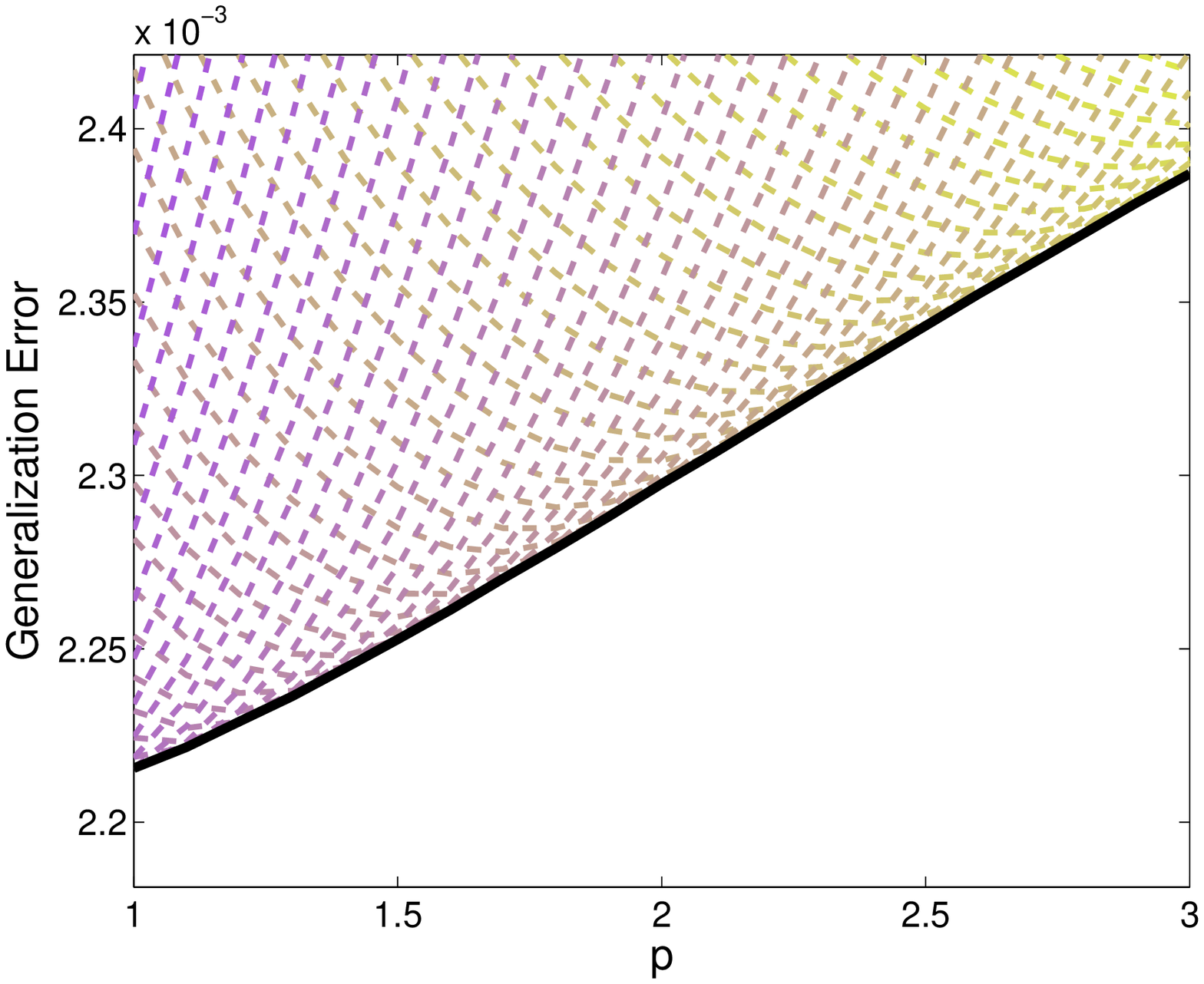}
}
\subfigure[Inhomogeneous Setting ($n=400$)]{
\includegraphics[width=7cm,clip]{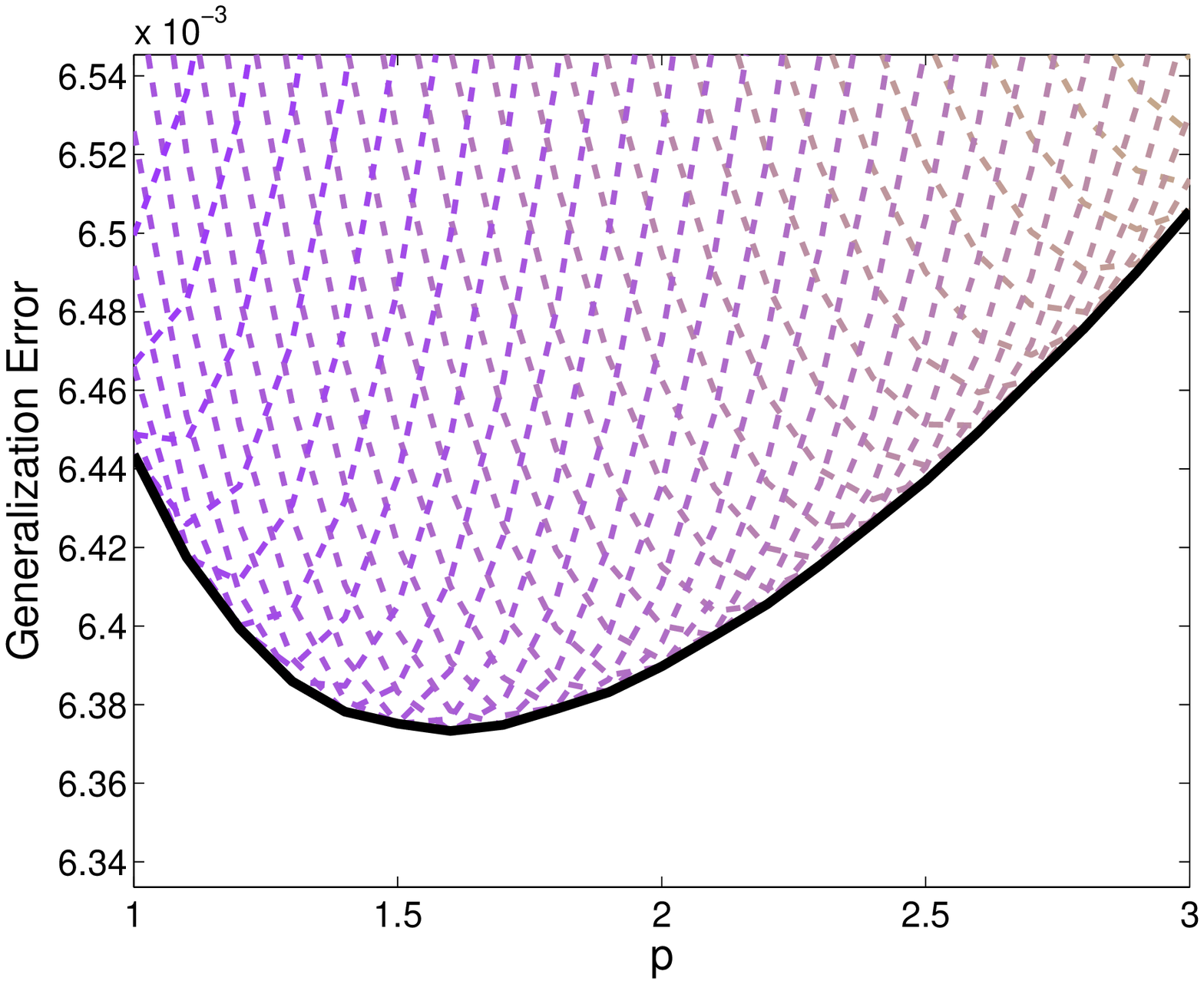}
}

\caption{The expected generalization error $\EE[\|\fhat - \fstar\|_{\LPi}^2]$ against the parameter $p$ for \ellp-MKL.
Each broken line corresponds to one regularization parameter. 
The bold solid line shows the best generalization error among all the regularization parameters.}
\label{fig:HomInhomComparison}
\end{center}
\end{figure}

Figure \ref{fig:HomInhomComparison} shows the average generalization errors 
in 
the homogeneous setting with (a) $n=200$ and (c) $n=400$, 
and
the inhomogeneous setting with (b) $n=200$ and (d) $n=400$.
Each broken line corresponds to one regularization parameter. 
The bold solid line shows the best (average) generalization error among all the regularization parameters.
We can see that in the homogeneous setting $\ell_1$-regularization shows the best performance,
on the other hand, in the inhomogeneous setting the best performance is achieved at $p>1$ for both $n=200$ and $400$. 
This experimental results beautifully matches the theoretical investigations.

\section{Generalization of loss function}
\label{sec:GeneralLoss}
Here we discuss how a general loss function other than squared loss can be involved into our analysis.
As in the standard local Rademacher complexity argument \citep{LocalRademacher}, 
we consider a class of loss functions that are Lipschitz continuous and strongly convex.
Suppose that the loss function $\psiloss:\Real \times \Real \to \Real$ satisfies Lipschitz continuity:
for all $R > 0$, there exists a constant $T(R)$ such that 
\begin{align} 
|\psiloss(y,f_1) - \psiloss(y,f_2)| \leq T(R)|f_1 - f_2|~~~~(\forall f_1,f_2 \in \Real~\text{such that $|f_1|,|f_2|\leq R$},~\forall y\in \Real).
\label{eq:psylossLipschitz}
\end{align} 
Moreover, suppose that, for all $y \in \Real$, $\psiloss(y,f)$ is a strongly convex with a modulus $\rho(R)>0$:
\begin{align} 
&\frac{\psiloss(y,f_1) + \psiloss(y,f_2)}{2} \geq \psiloss\left(y,\frac{f_1 + f_2}{2}\right) + \frac{\rho(R)}{2}|f_1 - f_2|^2 \notag \\
&~~(\forall f_1,f_2 \in \Real~\text{such that $|f_1|,|f_2|\leq R$}).
\label{eq:psylossStrongConvex}
\end{align} 
Some detailed discussions about these conditions and examples can be found in \cite{BartlettConvexity}.
Under the loss functions satisfying these properties, we obtain simplified bound
where some conditions can be omitted as follows:
\begin{itemize}
\item
We can remove 
the condition $\frac{4\phi \sqrt{n}}{\kminrho} \max\{\alpha_1^2,\beta_1^2,\frac{M\log(M)}{n} \}\eta(t') \leq \frac{1}{12}$,
\item 
The term $\exp(-t')$ is not needed in the tail probability.
\end{itemize}

To obtain a fast convergence rate on a general loss functions $\psiloss$, 
we move the regularization term in \Eqref{eq:primalLp} into a constraint,
and then consider the following optimization problem:
\begin{align}
\fhat = \sum_{m=1}^M \fhat_m = &
\mathop{\arg \min}_{\begin{subarray}{c}~f_m \in \calH_m~(m=1,\dots,M), \\ 
\|f \|_{\psi} \leq \Rn
\end{subarray}}
\frac{1}{n}\sum_{i=1}^N \psiloss \left(y_i, \sum_{m=1}^M f_m(x_i) \right),
\label{eq:primalLp_Lf}
\end{align}
where $\Rn$ is a regularization parameter.
The above optimization problem is essentially equivalent to the original formulation \eqref{eq:primalLp},
but 
by considering the constraint type regularization instead of the penalty type regularization
the theoretical analysis of statistical performance can be simplified.

We define $Pg$ as the expectation of a function $g:\Real \times \Real \to \Real$:
$$
Pg := \EE_{(X,Y)\sim P}[g(X,Y)].
$$
For notational simplicity, we write $P\psiloss(f) = P\psiloss(Y,f) = \EE_{(X,Y)\sim P}[\psiloss(Y,f(X))]$
for a function $f$.
We suppose there exists a minimizer for $P\psiloss(f)$ as follows.
\begin{Assumption}{\bf(Minimizer Existence Assumption)}\\ 
\label{ass:BasicAss_2}
There exists unique $\fstar = (\fstar_{1},\dots,\fstar_{M}) \in \calHtot$
such that 
\begin{flalign*}
\text{\rm \Assump{}} &&
\fstar = \sum_{m=1}^M \fstar_{m} = 
\mathop{\arg \min}_{
f_m \in \calH_m~(m=1,\dots,M)
}
P \psiloss \left(\sum_{m=1}^M f_m(X) \right).
&&
\end{flalign*}
\end{Assumption}
Note that, due to 
the incoherence assumption (Assumption \ref{ass:incoherence})
and 
the strong convexity \eqref{eq:psylossStrongConvex} of the loss function,
if there exists a minimizer, then that is automatically unique. 

To bound the convergence rate on a general loss function, 
it is convenient to utilize \emph{local Rademacher complexity} on $\psi$-norm ball.
Let $\calH_{\psi}^{(r)}(R) := \{f\in \calHtot \mid \|f\|_{\LPi}\leq r ,~\|f\|_{\psi} \leq R \}$.
Then the local Rademacher complexity of $\calH_{\psi}^{(r)}(R)$ is defined as
$$
R_n(\calH_{\psi}^{(r)}(R)) := \EE_{\{\sigma_i,x_i\}_{i=1}^n}\left[ \sup_{f \in \calH_{\psi}^{(r)}(R)} \frac{1}{n} \sum_{i=1}^n \sigma_i f(x_i) \right],
$$
where $\sigma_i \in \{\pm 1\}$ is the i.i.d. Rademacher random variable with $P(\sigma_i = 1) = P(\sigma_i = -1) = \frac{1}{2}$.
Evaluating the local Rademacher complexity is a key ingredient to show a fast convergence rate on a general loss function.
We obtain the following estimation of the local Rademacher complexity 
(the proof will be given in Appendix \ref{sec:DerivationLocalRademacher}).

\begin{Lemma}
\label{lemm:localRademacherComplexity}
Let $\{\cm\}_{m=1}^M$ be arbitrary positive reals.
Under Assumptions 2-5,
there exists a constant $\tilde{\phi}$ depending on $\{\sm\}_{m=1}^M,c,C_1$ such that 
for all $n$ satisfying  $\frac{\log(M)}{\sqrt{n}}\leq 1$ we have 
\begin{align*}
R_n(\calH_{\psi}^{(r)}(R)) \leq  
\tilde{\phi}
\left(\alpha_1 \frac{r}{\sqrt{\kminrho}} +  
\alpha_2 R 
+ \beta_1 \frac{r}{\sqrt{\kminrho}} +  
\beta_2 R 
+\sqrt{\frac{M\log(M)}{n}} \frac{r}{\sqrt{\kminrho}}\right).
\end{align*}
\end{Lemma}

Finally note that the supremum norm of $f$ with $\|f\|_{\psi} \leq \Rn$ can be bounded as 
$$
\|f\|_{\infty} \leq \sum_{m=1}^M \|f_m\|_{\infty}
\leq \sum_{m=1}^M \|f_m\hnorm{m}
\leq \|\boldone\|_{\psi^*} \|f \|_{\psi} \leq \|\boldone\|_{\psi^*} \Rn.
$$ 
Then, we obtain the {\it excess risk} bound as in the following theorem.
\begin{Theorem}
\label{th:convergencerateofLpMKL_Lipschitz}
Suppose Assumptions 2-5 and 7 are satisfied and 
the loss function $\psiloss$ satisfies 
the conditions \eqref{eq:psylossLipschitz} and \eqref{eq:psylossStrongConvex}. 
Let $\{\cm\}_{m=1}^M$ be arbitrary positive reals that can depend on $n$ and 
let $\bar{T} = T(\|\boldone\|_{\psi^*} \Rn)$ and $\bar{\rho} = \rho(\|\boldone\|_{\psi^*} \Rn)$.
Set $\Rn = \|\fstar\|_{\psi}$.
Then there exists a constant $\tilde{\phi}'$ depending on $\{\sm\}_{m=1}^M,c,C_1$
such that for all $n$ satisfying  $\frac{\log(M)}{\sqrt{n}}\leq 1$, 
we have 
\begin{align}
&
P(\psiloss(\fhat) - \psiloss(\fstar_{\Rn})) \notag \\
\leq 
&
\frac{\tilde{\phi}' \bar{\rho}}{\kminrho}
\left(\alpha_1^2 + \beta_1^2 + \frac{M\log(M)}{n} \right)
+
\tilde{\phi}' \frac{\bar{T}^2}{\bar{\rho}} 
\left[
\left(\frac{\alpha_2}{\alpha_1}\right)^2+ 
\left(\frac{\beta_2}{\beta_1}\right)^2
\right]
\|\fstar\|_{\psi}^2
\notag \\
&+ \frac{\{22 \bar{T} \|\boldone\|_{\psi^*} \Rn + 27 \bar{\rho}\}t}{n},
\label{eq:ConvRateInMainTh_Lipschitz}
\end{align}
with probability $1- \exp(- t)$.
\end{Theorem}
This can be shown by applying the bound of the local Rademacher complexity (Lemma \ref{lemm:localRademacherComplexity})
to Corollary 5.3 of \cite{LocalRademacher}\footnote{
In Corollary 5.3 of \cite{LocalRademacher}, the range of the function class is assumed to be included in the interval $[-1,1]$.
Here we utilize more general settings where the interval is $[-a,a]$ and $\|\boldone\|_{\psi^*} \Rn$ is substituted to $a$. 
See Lemma 9 of \cite{arXiv:Kloft+Gilles:2011}. }. 
Compared with the bound in \Eqref{eq:ConvRateInMainTh}, 
we notice that there is no $\exp(-t')$ term in the tail probability bound, and thus
we don't need the condition $\frac{4\phi \sqrt{n}}{\kminrho} \max\{\alpha_1^2,\beta_1^2,\frac{M\log(M)}{n} \}\eta(t') \leq \frac{1}{12}$.
Because of this, the range of $n$ where the error bound holds is relaxed compared with that in Theorem \ref{th:convergencerateofLpMKL}.
These simplifications are due to the Lipschitz continuity of the loss function.
In Theorem \ref{th:convergencerateofLpMKL}, 
we should have bounded the discrepancy between the empirical and population means of the squared loss:
$\frac{1}{n} \sum_{i=1}^n (\fhat(x_i) - \fstar(x_i))^2 - P (\fhat - \fstar)^2$.
Since the squared loss is not Lipschitz continuous,
we required an additional bound for that discrepancy using Assumption \ref{ass:linfbound} for the supremum norm,
and it was shown that that discrepancy is negligible at the cost of $\exp(-t')$ in the tail probability.
On the other hand, for Lipschitz continuous losses, we no longer need to bound such a quantity. 
Thus 
the tail probability loss $\exp(-t')$ is not induced.

Since the bound \eqref{eq:ConvRateInMainTh_Lipschitz} is basically same as Eq.\eqref{eq:ConvRateInMainTh},
we obtain the same discussions as in the previous sections.
For example, in the homogeneous setting, we obtain the following convergence bound.
\begin{Lemma}
\label{lem:smcmuniformbound_Lipschitz}
When $\sm = s~(\forall m)$ with some $0<s<1$,  
if we set
$\Rn = \|\fstar\|_{\psi}$,
then 
for all $n$ satisfying  $\frac{\log(M)}{\sqrt{n}} \leq 1$ and 
$n \geq  (\|\boldone\|_{\psi^*} \|\fstar\|_{\psi}/M)^{\frac{4s}{1-s}}$,
and for all $t \geq 1$,
we have 
\begin{align}
P(\psiloss(\fhat) - \psiloss(\fstar))
\leq 
C 
\left\{
M^{1-\frac{2\smm{}}{1+\smm{}}} n^{-\frac{1}{1+\smm{}}}(\|\boldone\|_{\psi^*} \|f^*\|_{\psi})^{\frac{2\smm{}}{1+\smm{}}}
 + \frac{M\log(M)}{n} 
+ \frac{t}{n}\right\}, \notag
\end{align}
with probability $1- \exp(- t)$ where $C$ is a constant depending on $\tilde{\phi}'$, $\kminrho$, $\rho(\|\boldone\|_{\psi^*} \Rn)$, and
$T(\|\boldone\|_{\psi^*} \Rn)$.
\end{Lemma}

\section{Conclusion and Future Work} 
We have shown a unifying framework to derive the learning rate of MKL with arbitrary mixed-norm-type regularization.
To analyze the general result, we considered two situations: homogeneous settings and inhomogeneous settings.
We have seen that the convergence rate of $\ell_p$-MKL obtained in homogeneous settings is tighter and 
requires less restrictive condition than existing results.
We have also shown convergence rates of some examples (elasticnet-MKL and VSKL), and 
proved the derived learning rate is minimax optimal when $\psi$-norm is isotropic.
An interesting consequence was that $\ell_1$-regularization is optimal among all isotropic $\psi$-norm regularization in homogeneous settings.
In the analysis of inhomogeneous settings, 
we have shown that the dense type regularization can outperform the sparse $\ell_1$-regularization
using analytically obtained bounds and numerically computed bounds.
We observed that our bound well explains the experimental results favorable for dense type MKL.
Finally we numerically investigated the generalization errors of \ellp-MKL in a homogeneous setting and an inhomogeneous setting. 
The numerical experiments supported the theoretical findings that 
$\ell_1$-regularization is optimal in homogeneous settings 
but, on the other hand, dense type regularizations are preferred in inhomogeneous settings. 
This is the first result that suggests that 
the inhomogeneity of the complexities of RKHSs well justifies the favorable performances for dense type MKL. 

An interesting future work is about the $\frac{M \log(M)}{n}$ term appeared in the bound \Eqref{eq:convergenceRateSimple}. 
Because of this term, our bound is $O(M\log(M))$ with respect to $M$ while in the existing work that is 
$O(\sqrt{\log(M)} \vee M^{1-\frac{1}{p}})$ for \ellp-MKL. 
Therefore our bound is not tight in the global bound regime ($n\leq M^{\frac{2}{p}} R_p^{-2}\log(M)^{\frac{1+s}{s}}$
for \ellp-MKL).
It is an interesting issue to clarify whether the term $\frac{M \log(M)}{n}$ can be replaced by other tighter bounds or not.
To do so, it might be helpful to combine our technique developed in this paper and 
that developed by \cite{arXiv:Kloft+Gilles:2011} where the local Rademacher complexity for $\ell_p$-MKL is derived.  

\section*{Acknowledgement}
We would like to thank Marius Kloft, Gilles Blanchard, Ryota Tomioka and Masashi Sugiyama 
for suggestive discussions. 
This work was partially supported by MEXT Kakenhi 22700289 and 
the Aihara Project, the FIRST program from JSPS, initiated by CSTP.

\appendix 
\section{Relation between Entropy Number and Spectral Condition}
Associated with the $\epsilon$-covering number, 
the {\it $i$-th entropy number} $e_i(\calH_m \to \LPi) $ is defined as the infimum over all $\varepsilon>0$ for which $N(\varepsilon,\mathcal{B}_{\calH_m},\LPi) \leq 2^{i-1}$.
If the spectral assumption (A3) and the boundedness assumption (A\ref{ass:kernelbound}) hold, 
the relation \eqref{eq:coveringcondition} implies that the $i$-th entropy number is bounded as 
\begin{align}
\label{eq:entropycondition}
e_i(\calH_m \to \LPi) \leq C i^{- \frac{1}{2s}},
\end{align}
where $C$ is a constant.
To bound empirical process a bound of the entropy number with respect to the empirical distribution is needed.
The following proposition gives an upper bound of that (see Corollary 7.31 of \cite{Book:Steinwart:2008}, for example).
\begin{Proposition}
\label{prop:upperboundofe}
If there exists constants $0<s<1$ and $C \geq 1$ such that 
$e_i(\calH_m \to \LPi) \leq C i^{- \frac{1}{2s}}$, then  there exists a constant $c_s > 0$ only depending on $s$ such that 
\begin{align*}
\EE_{D_n \sim \Pi^n}[e_i(\calH_m \to L_2(D_n))] \leq c_s C (\min(i,n))^{\frac{1}{2s}} i^{-\frac{1}{s}},
\end{align*}
in particular 
$
\EE_{D_n \sim \Pi^n}[e_i(\calH_m \to L_2(D_n))] \leq c_s C  i^{-\frac{1}{2s}}.
$
\end{Proposition}

\section{Basic Propositions}
The following two propositions are keys to prove Theorem \ref{th:convergencerateofLpMKL}. 
Let $\{\sigma_i\}_{i=1}^n$ be i.i.d. Rademacher random variables, i.e., $\sigma_i \in \{\pm 1\}$ and $P(\sigma_i=1) = P(\sigma_i=-1)=\frac{1}{2}$.  
\begin{Proposition}{\rm \bf \cite[Theorem 7.16]{Book:Steinwart:2008}}
\label{prop:localFmsBound}
Let $\calB_{\sigma,a,b} \subset \calH_m$ be a set such that $\calB_{\sigma,a,b} = \{ f_m \in \calH_m \mid \|f_m\|_{\LPi}\leq \sigma, \|f_m\hnorm{m} \leq a, \|f_m\|_{\infty} \leq b\} $.
Assume that there exist constants $0<s<1$ and $0 < \tilde{c}_s$ such that 
\begin{align*}
\EE_{D_n}[ e_i(\calH_m \to L_2(D_n))] \leq \tilde{c}_s i^{-\frac{1}{2s}}.
\end{align*}
Then there exists a constant $C_s'$ depending only $s$ such that 
\begin{align}
\EE\left[\sup_{f_m \in \calB_{\sigma,a,b}} \left| \frac{1}{n}\sum_{i=1}^n \sigma_i f_m(x_i) \right|\right] \leq 
C_s' \left( \frac{ \sigma^{1-s} (\tilde{c}_s a)^s}{\sqrt{n}} \vee (\tilde{c}_s a)^{\frac{2s}{1+s}} b^{\frac{1-s}{1+s}} n^{-\frac{1}{1+s}} \right). 
\label{eq:localFmsBound}
\end{align}
\end{Proposition}

\begin{Proposition}{\rm \bf (Talagrand's Concentration Inequality \cite{Talagrand2,BousquetBenett})}
\label{prop:TalagrandConcent}
Let $\calG$ be a function class on $\calX$ that is separable with respect to $\infty$-norm, and 
$\{x_i\}_{i=1}^n$ be i.i.d. random variables with values in $\calX$.
Furthermore, let $B\geq 0$ and $U\geq 0$ be 
$B := \sup_{g \in \calG} \EE[(g-\EE[g])^2]$ and $U := \sup_{g \in \calG} \|g\|_{\infty}$,
then there exists a universal constant $K$ such that, for $Z := \sup_{g\in \calG}\left|\frac{1}{n} \sum_{i=1}^n g(x_i) - \EE[g] \right|$, we have
\begin{align*}
P\left( Z \geq K\left[\EE[Z] + \sqrt{\frac{B t}{n}} + \frac{U t}{n} \right] \right) \leq e^{-t}.
\end{align*}
\end{Proposition}

\section{Proof of Theorem \ref{th:convergencerateofLpMKL}}
\label{sec:ProofMainTh}


Let $\cm >0~(m=1,\dots,M)$ be arbitrary positive reals.
Given $\{\cm\}_{m=1}^M$, we determine $\Uns(f_m)$ as follows:
\begin{align*}
\Uns(f_m) := & 3\left( 
\frac{\cm^{-\sm}}{\sqrt{n}} \vee 
\frac{\cm^{-\frac{\sm(3-\sm)}{1+\sm}}}{n^{\frac{1}{1+\sm}}}  \right) 
\left(\|f_m\|_{\LPi} +  \sm \cm \|f_m\hnorm{m}\right)
+ 
\sqrt{\frac{\log(M)}{n}}\|f_m\|_{\LPi} .
\end{align*} 
It is easy to see $\Uns(f_m)$ is an upper bound of the quantity  
$\frac{\|f_m\|_{\LPi}^{1-\sm}\|f_m \hnorm{m}^{\sm}}{\sqrt{n}} \vee 
\frac{\|f_m \|_{\LPi}^{\frac{(1-\sm)^2}{1+\sm}} \|f_m \hnorm{m}^{\frac{\sm(3-\sm)}{1+\sm}}}{n^{\frac{1}{1+\sm}}}$
(this corresponds to the RHS of \Eqref{eq:localFmsBound})
because 
\begin{align}
\frac{\|f_m\|_{\LPi}^{1-\sm}\|f_m \hnorm{m}^{\sm}}{\sqrt{n}} 
&
=
\frac{\cm^{1-\sm}}{\sqrt{n}} \left(\frac{\|f_m\|_{\LPi}}{\cm}\right)^{1-\sm}  \|f_m \hnorm{m}^{\sm} \notag \\
&\opleq{(Young)}
\frac{\cm^{1-\sm}}{\sqrt{n}} \left( (1-\sm) \frac{\|f_m\|_{\LPi}}{\cm}  + \sm \|f_m \hnorm{m} \right) \notag \\
&
\leq
\frac{\cm^{-\sm}}{\sqrt{n}} \left(  \|f_m\|_{\LPi}  + \sm \cm \|f_m \hnorm{m} \right),
\label{eq:sdecomptwo}
\end{align}
where we used Young's inequality $a^{1-\sm}b^{\sm} \leq (1-\sm) a + \sm b$ in the second line, and similarly we obtain
\begin{align*}
\frac{\|f_m \|_{\LPi}^{\frac{(1-\sm)^2}{1+\sm}} \|f_m \hnorm{m}^{\frac{\sm(3-\sm)}{1+\sm}}}{n^{\frac{1}{1+\sm}}}
&\leq 
\frac{\cm^{-\frac{\sm(3-\sm)}{1+\sm}}}{n^{\frac{1}{1+\sm}}} 
\left(\|f_m\|_{\LPi} + \frac{\sm(3-\sm)}{1+\sm} \cm \|f_m \hnorm{m} \right) \\
&\leq 
3 \frac{\cm^{-\frac{\sm(3-\sm)}{1+\sm}}}{n^{\frac{1}{1+\sm}}} 
\left(\|f_m\|_{\LPi} + \sm \cm \|f_m \hnorm{m} \right),
\end{align*}
where we used $\frac{\sm(3-\sm)}{1+\sm} \leq 3 \sm$ in the last inequality.

Now we define 
\begin{align*}
\phi := \max \left( KL \left[2 \tilde{C}_*   +  1 + C_1  \right] ,K\left[2C_1 \tilde{C}_*   +  C_1 + C_1^2  \right] \right),
\end{align*}
where $\tilde{C}^*$ is a constant defined later in Lemma \ref{lemm:basicuniformlemm}, $C_1$ is the one introduced 
in Assumption \ref{ass:linfbound}, $K$ is the universal constant appeared in Talagrand's concentration inequality 
(Proposition \ref{prop:TalagrandConcent}) and $L$ is the one introduced in Assumption \ref{ass:BasicAss} 
to bound the magnitude of noise.
Remind the definition of $\eta(t)$:
$$
\eta(t) := \eta_{n}(t) = \max(1,\sqrt{t},t/\sqrt{n}).
$$
We define events $\scrE_1(t)$ and $\scrE_2(t')$ as 
\begin{align}
&
\scrE_1(t) = \left\{ 
\left| \frac{1}{n}\sum_{i=1}^n \epsilon_i f_m(x_i)  \right|  \leq 
\phi 
\Uns(f_m)\eta(t) ,~\forall f_m \in \calH_m~(m=1,\dots,M)  \right\}, 
\label{eq:defE1}\\
&
\scrE_2(t') = \Bigg\{ 
\left| \textstyle \left\|\sum_{m=1}^M f_m \right\|_n^2 - \left\|\sum_{m=1}^M f_m \right\|_{\LPi}^2 \right| \leq 
\phi
\sqrt{n}
\left(\sum_{m=1}^M \Uns(f_m)\right)^2\eta(t'),\notag \\
&~~~~~~~~~~~~~~~~~~~~~~~~~~~~~~~~~~~\forall f_m \in \calH_m~(m=1,\dots,M) \Bigg\}.
\label{eq:defE2}
\end{align}

Using Lemmas \ref{lemm:uniformRatioBoundOnM} and \ref{th:SquareBound} that will be shown in Appendix \ref{sec:BoundProbE1E2},
we see that the events $\scrE_1(t)$ and $\scrE_2(t')$ occur with probability no less than $1 - \exp(-t)$ and $1-\exp(-t')$ respectively as in the following Lemma.
\begin{Lemma}
\label{lemm:E1E2probBound}
Under the Basic Assumption (Assumption \ref{ass:BasicAss}), the Spectral Assumption (Assumption \ref{eq:specass}) 
and the Embedded Assumption (Assumption \ref{ass:linfbound}), the probabilities of $\scrE_1(t)$ and $\scrE_2$ are bounded as  
$$
P(\scrE_1(t)) \geq 1 - \exp(-t),~~~P(\scrE_2(t')) \geq 1 - \exp(-t').
$$
\end{Lemma}
\begin{proof}
Lemma \ref{th:SquareBound} immediately 
gives $P(\scrE_1(t)) \geq 1 - \exp(-t)$ by noticing $\bar{\phi}$ in the statement of Lemma \ref{th:SquareBound} satisfies $\bar{\phi} \leq \phi$.
Moreover, since 
$\bar{\phi}'$ in the statement of Lemma \ref{lemm:uniformRatioBoundOnM} satisfies $\bar{\phi}' \leq \phi$,
we have $P(\scrE_2(t')) \geq 1 - \exp(-t')$ by Lemma \ref{lemm:uniformRatioBoundOnM}.
\end{proof}

Remind the definition \eqref{eq:defalphabeta} of $\alpha_1,\alpha_2,\beta_1,\beta_2$:
\begin{align}
&\alpha_1 = 3 \left( \sum_{m=1}^M \frac{\cm^{-2\sm}}{n}\right)^{\frac{1}{2}},~~
\alpha_2 = 3  \left\| \left(\frac{\sm \cm^{1-\sm}}{\sqrt{n}}\right)_{m=1}^M \right\|_{\psi^*}, \notag \\
&
\beta_1 = 3 \left( \sum_{m=1}^M \frac{\cm^{-\frac{2\sm(3-\sm)}{1+\sm}}}{n^{\frac{2}{1+\sm}}} \right)^{\frac{1}{2}},~~
\beta_2 = 3 \left\| \left(\frac{\sm \cm^{\frac{(1-\sm)^2}{1+\sm}}}{n^{\frac{1}{1+\sm}}} \right)_{m=1}^M \right\|_{\psi^*},
\end{align}
for given reals $\{\cm\}_{m=1}^M$. 
The following theorem immediately gives Theorem \ref{th:convergencerateofLpMKL}.

\begin{Theorem}
\label{th:convergencerateofLpMKLdet}
Suppose Assumptions 1-4 are satisfied. 
Let $\{\cm\}_{m=1}^M$ be arbitrary positive reals that can depend on $n$, and 
assume  
$\lambdaone \geq \left(\frac{\alpha_2}{\alpha_1}\right)^2 + \left(\frac{\beta_2}{\beta_1}\right)^2$.
Then for all $n$ and $t'$ that satisfy  $\frac{\log(M)}{\sqrt{n}}\leq 1$ and 
$\frac{4\phi \sqrt{n}}{\kminrho} \max\{\alpha_1^2,\beta_1^2,\frac{M\log(M)}{n} \}\eta(t') \leq \frac{1}{12}$
and for all $t \geq 1$,
we have 
\begin{align*}
&\|\fhat - \fstar\|_{\LPi}^2 
\leq 
\frac{24 \eta(t)^2 \phi^2}{\kminrho} \left(\alpha_1^2 + \beta_1^2 + \frac{M\log(M)}{n} \right)
+
4 \lambdaone  \|\fstar \|_{\psi}^2.
\end{align*}
with probability $1- \exp(- t) - \exp(-t')$.
\end{Theorem}

\begin{proof} [Proof of Theorem \ref{th:convergencerateofLpMKLdet}]
By the assumption of the theorem, 
we can assume Lemma \ref{lemm:E1E2probBound} holds, that is,
the event $\scrE_1(t) \cap \scrE_2(t')$ occurs with probability $1-\exp(-t) - \exp(-t')$.
Below we discuss on the event $\scrE_1(t) \cap \scrE_2(t')$.

Since $y_i = \fstar(x_i) + \epsilon_i$, we have 
\begin{align*}
&\|\fhat - \fstar \|_{\LPi}^2 + \lambdaone \| \fhat \|_{\psi}^2 
\notag \\
\leq 
&( \|\fhat - \fstar \|_{\LPi}^2 - \|\fhat - \fstar \|_{n}^2 ) + 
\frac{2}{n}\sum_{i=1}^n \sum_{m=1}^M\epsilon_i (\fhat_m(x_i) - \fstar_m(x_i)) + \lambdaone \left\| \fstar\right\|_{\psi}^2.
\end{align*}
Here on the event $\scrE_2(t')$, the above inequality gives 
\begin{align}
&\|\fhat - \fstar \|_{\LPi}^2 + \lambdaone  \|\fhat \|_{\psi}^2 
\notag \\
\leq 
&  \phi \sqrt{n} \left(\sum_{m=1}^M \Uns(\fhat_m - \fstar_m) \right)^2 \eta(t') \! + \!
\frac{2}{n}\sum_{i=1}^n \sum_{m=1}^M\epsilon_i (\fhat_m(x_i) - \fstar_m(x_i)) + \lambdaone \|\fstar \|_{\psi}^2.
\label{eq:basicineqLast}
\end{align}

Before we prove the statements, we show an upper bound of $\sum_{m=1}^M \Uns(f_m)$
required in the proof. By definition, we have 
\begin{align}
&\Uns(f_m) \notag \\
= & 3\left( 
\frac{\cm^{-\sm}}{\sqrt{n}} \vee 
\frac{\cm^{-\frac{\sm(3-\sm)}{1+\sm}}}{n^{\frac{1}{1+\sm}}}  \right) 
\left(\|f_m\|_{\LPi} +  \sm \cm \|f_m\hnorm{m}\right)
+ 
\sqrt{\frac{\log(M)}{n}}\|f_m\|_{\LPi} \notag \\
\leq 
&
3\frac{\cm^{-\sm}}{\sqrt{n}} \left(\|f_m\|_{\LPi} +  \sm \cm \|f_m\hnorm{m}\right) + 
3 \frac{\cm^{-\frac{\sm(3-\sm)}{1+\sm}}}{n^{\frac{1}{1+\sm}}} \left(\|f_m\|_{\LPi} +  \sm \cm \|f_m\hnorm{m}\right)  \\
&+ 
\sqrt{\frac{\log(M)}{n}}\|f_m\|_{\LPi}.
\label{eq:UnsDecomp}
\end{align} 
Now the sum of the first term is bounded as
\begin{align*}
& \sum_{m=1}^M 3\frac{\cm^{-\sm}}{\sqrt{n}} \left(\|f_m\|_{\LPi} +  \sm \cm \|f_m\hnorm{m}\right) \\
=
&
3 \sum_{m=1}^M \frac{\cm^{-\sm}}{\sqrt{n}} \|f_m\|_{\LPi} +  3 \sum_{m=1}^M \frac{\sm \cm^{1-\sm}}{\sqrt{n}} \|f_m\hnorm{m} \\
\leq
&
3 \left( \sum_{m=1}^M \frac{\cm^{-2\sm}}{n}\right)^{\frac{1}{2}}  \left( \sum_{m=1}^M \|f_m\|_{\LPi}^2\right)^{\frac{1}{2}} +  
3  \left\| \left(\frac{\sm \cm^{1-\sm}}{\sqrt{n}}\right)_{m=1}^M \right\|_{\psi^*} \|f \|_{\psi},
\end{align*}
where we used Cauchy-Schwarz inequality and the duality of the norm in the last inequality.
The sum of the second term of the RHS of \Eqref{eq:UnsDecomp} is bounded as 
\begin{align*}
&\sum_{m=1}^M  3 \frac{\cm^{-\frac{\sm(3-\sm)}{1+\sm}}}{n^{\frac{1}{1+\sm}}} \left(\|f_m\|_{\LPi} +  \sm \cm \|f_m\hnorm{m}\right)   \\
=&3 \sum_{m=1}^M  \frac{\cm^{-\frac{\sm(3-\sm)}{1+\sm}}}{n^{\frac{1}{1+\sm}}} \|f_m\|_{\LPi} + 3 \sum_{m=1}^M \frac{\sm \cm^{\frac{(1-\sm)^2}{1+\sm}}}{n^{\frac{1}{1+\sm}}}  \|f_m\hnorm{m}   \\
\leq
&3 \left( \sum_{m=1}^M \frac{\cm^{-\frac{2\sm(3-\sm)}{1+\sm}}}{n^{\frac{2}{1+\sm}}} \right)^{\frac{1}{2}} \left(\sum_{m=1}^M \|f_m\|_{\LPi}^2 \right)^{\frac{1}{2}} +  
3 \left\| \left(\frac{\sm \cm^{\frac{(1-\sm)^2}{1+\sm}}}{n^{\frac{1}{1+\sm}}} \right)_{m=1}^M \right\|_{\psi^*}  \|f \|_{\psi},
\end{align*}
where we used Cauchy-Schwarz inequality and the duality of the norm in the last inequality.
Finally we have the following bound of the third term of the RHS of \Eqref{eq:UnsDecomp}:
\begin{align*}
\sum_{m=1}^M \sqrt{\frac{\log(M)}{n}}\|f_m\|_{\LPi} \leq \sqrt{\frac{M\log(M)}{n}} \left(\sum_{m=1}^M \|f_m\|_{\LPi}^2 \right)^{\frac{1}{2}}.
\end{align*}
Combine these inequalities and the relation $\sum_{m=1}^M \|f_m\|_{\LPi}^2 \leq \frac{1}{\kminrho} \|f\|_{\LPi}^2 $ (Assumption \ref{ass:incoherence}) to obtain 
\begin{align}
&\sum_{m=1}^M \Uns(f_m) \notag \\
\leq 
&
3 \left( \sum_{m=1}^M \frac{\cm^{-2\sm}}{n}\right)^{\frac{1}{2}}  \frac{\|f \|_{\LPi}}{\sqrt{\kminrho}} +  
3  \left\| \left(\frac{\sm \cm^{1-\sm}}{\sqrt{n}}\right)_{m=1}^M \right\|_{\psi^*} \|f \|_{\psi} \notag \\
&
+3 \left( \sum_{m=1}^M \frac{\cm^{-\frac{2\sm(3-\sm)}{1+\sm}}}{n^{\frac{2}{1+\sm}}} \right)^{\frac{1}{2}} \frac{\|f \|_{\LPi}}{\sqrt{\kminrho}} +  
3 \left\| \left(\frac{\sm \cm^{\frac{(1-\sm)^2}{1+\sm}}}{n^{\frac{1}{1+\sm}}} \right)_{m=1}^M \right\|_{\psi^*}  \|f \|_{\psi} \notag \\
&
+\sqrt{\frac{M\log(M)}{n}} \frac{\|f \|_{\LPi}}{\sqrt{\kminrho}}.
\label{eq:UnsmBoundAlphaBeta}
\end{align}
Then by the definition \eqref{eq:defalphabeta} of $\alpha_1,\alpha_2,\beta_1,\beta_2$, we have
\begin{align}
&\sum_{m=1}^M \Uns(f_m) \notag \\
\leq 
&\alpha_1 \frac{\|f \|_{\LPi}}{\sqrt{\kminrho}} +  
\alpha_2 \|f \|_{\psi} 
+ \beta_1 \frac{\|f \|_{\LPi}}{\sqrt{\kminrho}} +  
\beta_2 \|f \|_{\psi} 
+\sqrt{\frac{M\log(M)}{n}} \frac{\|f \|_{\LPi}}{\sqrt{\kminrho}}.
\label{eq:Unbounds}
\end{align}


\noindent{\it Step 1.}

By \Eqref{eq:Unbounds}, 
the first term on the RHS of \Eqref{eq:basicineqLast} can be upper bounded as 
\begin{align*}
& \phi \sqrt{n} \left( \sum_{m=1}^M \Uns(\fhat_m - \fstar_m) \right)^2 \eta(t')  \\
\leq 
& 
4 \phi \sqrt{n} \Big(
\alpha_1^2  \frac{\| \fhat - \fstar \|_{\LPi}^2}{\kminrho}
+ 
\alpha_2^2 \|\fhat - \fstar \|_{\psi}^2 
+
\beta_1^2 \frac{\| \fhat - \fstar \|_{\LPi}^2}{\kminrho} 
+ \\
&~~~~~~~~~~~~~~~\beta_2^2 \|\fhat - \fstar \|_{\psi}^2 
+
\frac{M\log(M)}{n} \frac{\|\fhat - \fstar \|_{\LPi}^2}{\kminrho}
\Big) \eta(t') \\
\leq 
& 
\frac{4 \phi \sqrt{n}}{\kminrho}
\alpha_1^2 \eta(t') \left( \| \fhat - \fstar \|_{\LPi}^2
+ 
\left(\frac{\alpha_2}{\alpha_1}\right)^2 \|\fhat - \fstar \|_{\psi}^2 \right) \\
&
+
\frac{4 \phi \sqrt{n}}{\kminrho}
\beta_1^2 \eta(t') \left( 
\| \fhat - \fstar \|_{\LPi}^2
+
\left(\frac{\beta_2}{\beta_1}\right)^2 \|\fhat - \fstar \|_{\psi}^2 \right) \\
& +
\frac{4 \phi \sqrt{n}}{\kminrho} \frac{M\log(M)}{n}\eta(t') \|\fhat - \fstar \|_{\LPi}^2.
\end{align*}
By assumption, we have $\frac{4\phi \sqrt{n}}{\kminrho} \max\{\alpha_1^2,\beta_1^2,\frac{M\log(M)}{n} \} \eta(t') \leq \frac{1}{12}$.
Hence the RHS of the above inequality is bounded by 
\begin{align}
 &\phi \sqrt{n} \left( \sum_{m=1}^M \Uns(\fhat_m - \fstar_m) \right)^2 \eta(t') \notag \\
\leq & 
\frac{1}{4} \left\{\|\fhat - \fstar \|_{\LPi}^2 + \left[\left(\frac{\alpha_2}{\alpha_1}\right)^2 + \left(\frac{\beta_2}{\beta_1}\right)^2 \right]
 \|\fhat - \fstar \|_{\psi}^2 \right\}.
\label{eq:firsttermbound2}
\end{align}

~\\
\noindent{\it Step 2.}
On the event $\scrE_1(t)$, we have 
\begin{align}
&\phantom{\leq} \frac{2}{n}\sum_{i=1}^n \sum_{m=1}^M\epsilon_i (\fhat_m(x_i) - \fstar_m(x_i))  
\leq 
2\sum_{m=1}^M \eta(t) \phi \Uns(\fhat_m - \fstar_m) \notag \\
&\leq 
2\eta(t) \phi \Bigg[ \alpha_1 \frac{\|\fhat - \fstar \|_{\LPi}}{\sqrt{\kminrho}} +  
\alpha_2 \|\fhat - \fstar \|_{\psi} 
+ \beta_1 \frac{\|\fhat - \fstar \|_{\LPi}}{\sqrt{\kminrho}} +  
\beta_2 \|\fhat - \fstar \|_{\psi} \notag \\
&~~~~
+\sqrt{\frac{M\log(M)}{n}} \frac{\|\fhat - \fstar\|_{\LPi}}{\sqrt{\kminrho}} \Bigg] 
~~~~~~~~~~~~~~~~~~~~~~~~~~~~~~~~~~~(\because \text{Eq.\eqref{eq:UnsmBoundAlphaBeta}}) \notag \\
&\leq 
2\frac{\eta(t) \phi  \alpha_1}{\sqrt{\kminrho}} \left( 
\|\fhat - \fstar \|_{\LPi} +  \frac{\alpha_2}{\alpha_1} \|\fhat - \fstar \|_{\psi} \right)
+
2\frac{ \eta(t) \phi \beta_1}{\sqrt{\kminrho}} \left(
\|\fhat - \fstar \|_{\LPi} +  
\frac{\beta_2}{\beta_1} \|\fhat - \fstar \|_{\psi} \right) \notag\\
&~~~~
+2\frac{\eta(t) \phi}{\sqrt{\kminrho}} \sqrt{\frac{M\log(M)}{n}}\|\fhat - \fstar\|_{\LPi}\notag\\
&
\leq 
\frac{12 \eta(t)^2 \phi^2  \alpha_1^2}{\kminrho} 
+ \frac{1}{24}\left( 
\|\fhat - \fstar \|_{\LPi} +  \frac{\alpha_2}{\alpha_1} \|\fhat - \fstar \|_{\psi} \right)^2  \notag\\
&~~~~~~+
\frac{12 \eta(t)^2 \phi^2 \beta_1^2}{\kminrho} +
\frac{1}{24}\left(
\|\fhat - \fstar \|_{\LPi} +  
\frac{\beta_2}{\beta_1} \|\fhat - \fstar \|_{\psi} \right)^2 \notag\\
&~~~~~~
+\frac{6 \eta(t)^2 \phi^2}{\kminrho} \frac{M\log(M)}{n} + \frac{1}{12}\|\fhat - \fstar\|_{\LPi}^2 \notag\\
&\leq 
\frac{12 \eta(t)^2 \phi^2  \alpha_1^2}{\kminrho} 
+ \frac{1}{12}\left[ 
\|\fhat - \fstar \|_{\LPi}^2 +  \left(\frac{\alpha_2}{\alpha_1}\right)^2 \|\fhat - \fstar \|_{\psi}^2 \right] \notag \\
&~~~~~~+
\frac{12 \eta(t)^2 \phi^2 \beta_1^2}{\kminrho} +
\frac{1}{12}\left[
\|\fhat - \fstar \|_{\LPi}^2 +  
\left(\frac{\beta_2}{\beta_1}\right)^2 \|\fhat - \fstar \|_{\psi}^2 \right] \notag\\
&~~~~~~
+\frac{6 \eta(t)^2 \phi^2}{\kminrho} \frac{M\log(M)}{n} + \frac{1}{12}\|\fhat - \fstar\|_{\LPi}^2 \notag \\
&\leq
\frac{12 \eta(t)^2 \phi^2}{\kminrho} \left(\alpha_1^2 + \beta_1^2 + \frac{M\log(M)}{n} \right)
+
\frac{1}{4}\left\{
\|\fhat - \fstar \|_{\LPi}^2 +  
\left[\left(\frac{\alpha_2}{\alpha_1}\right)^2 + \left(\frac{\beta_2}{\beta_1}\right)^2 \right] \|\fhat - \fstar \|_{\psi}^2 \right\}.
\label{eq:secondtermbound}
\end{align}

~\\
\noindent{\it Step 3.}

Substituting the inequalities 
\eqref{eq:firsttermbound2} and  \eqref{eq:secondtermbound} 
to \Eqref{eq:basicineqLast}, we obtain
\begin{align}
&\|\fhat - \fstar\|_{\LPi}^2 
+ \lambdaone  \|\fhat \|_{\psi}^2 \notag\\
\leq &
\frac{12 \eta(t)^2 \phi^2}{\kminrho} \left(\alpha_1^2 + \beta_1^2 + \frac{M\log(M)}{n} \right)
+
\frac{1}{2}
\left\{
\|\fhat - \fstar \|_{\LPi}^2 +  
\left[\left(\frac{\alpha_2}{\alpha_1}\right)^2 + \left(\frac{\beta_2}{\beta_1}\right)^2 \right] \|\fhat - \fstar \|_{\psi}^2 \right\} \notag\\
&+ \lambdaone  \|\fstar \|_{\psi}^2.
\label{eq:combineq1}
\end{align}
Now, by the triangular inequality, the term $\|\fhat - \fstar \|_{\psi}^2$ can be bounded as 
\begin{align*}
&\|\fhat - \fstar \|_{\psi}^2
\leq
\left( \|\fhat\|_{\psi} + \| \fstar \|_{\psi} \right)^2
\leq 
2 \left(\|\fhat\|_{\psi}^2 + \| \fstar \|_{\psi}^2 \right). 
\end{align*}
Thus, when $\lambdaone \geq \left(\frac{\alpha_2}{\alpha_1}\right)^2 + \left(\frac{\beta_2}{\beta_1}\right)^2$, 
\Eqref{eq:combineq1} yields
\begin{align}
&\frac{1}{2}\|\fhat - \fstar\|_{\LPi}^2 
\leq 
\frac{12 \eta(t)^2 \phi^2}{\kminrho} \left(\alpha_1^2 + \beta_1^2 + \frac{M\log(M)}{n} \right)
+
2 \lambdaone  \|\fstar \|_{\psi}^2.\notag
\end{align}
Therefore by multiplying 2 to both sides, we have 
\begin{align}
&\|\fhat - \fstar\|_{\LPi}^2 
\leq 
\frac{24 \eta(t)^2 \phi^2}{\kminrho} \left(\alpha_1^2 + \beta_1^2 + \frac{M\log(M)}{n} \right)
+
4 \lambdaone  \|\fstar \|_{\psi}^2.\notag
\end{align}
This gives the assertion. 
\end{proof}

\section{Bounding the Probabilities of $\scrE_1(t)$ and $\scrE_2(t')$}
\label{sec:BoundProbE1E2}
Here we derive bounds of the probabilities of the events $\scrE_1(t)$ and $\scrE_2(t')$ (see \Eqref{eq:defE1} and \Eqref{eq:defE2} for their definitions).
The goal of this section is to derive Lemmas \ref{lemm:uniformRatioBoundOnM} and \ref{th:SquareBound}.

Using Propositions \ref{prop:TalagrandConcent} and \ref{prop:localFmsBound}, 
we obtain the following ratio type uniform bound.
\begin{Lemma}
\label{lemm:uniformratiobound}
Under the Spectral Assumption (Assumption \ref{eq:specass}) and the Embedded Assumption (Assumption \ref{ass:linfbound}), 
there exists a constant $C_{\sm}$ depending only on $\sm$, $c$ and $C_1$ such that 
\begin{align}
\EE\left[\sup_{f_m \in \calH_m: \|f_m \hnorm{m}=1} 
\frac{|\frac{1}{n}\sum_{i=1}^n \sigma_i f_m(x_i)|}{\Uns(f_m)} \right] 
\leq C_{\sm}.
\notag
\end{align}
\end{Lemma}

\begin{proof} [Proof of Lemma \ref{lemm:uniformratiobound}]
Let $\calH_m(\delta) := \{ f_m \in \calH_m \mid \|f_m\hnorm{m} = 1, \|f_m\|_{\LPi} \leq \delta \}$ and $z = 2^{1/\sm} > 1$.
Define $\tau := \sm \cm$. Then 
by combining Propositions \ref{prop:upperboundofe} and \ref{prop:localFmsBound} with Assumption \ref{ass:linfbound}, we have 
\begin{align*}
&
\EE\left[\sup_{f_m \in \calH_m: \|f_m \hnorm{m}=1} 
\frac{|\frac{1}{n}\sum_{i=1}^n \sigma_i f_m(x_i)|}{\Uns(f_m)} \right]  \\
\leq 
&
\EE\left[\sup_{f_m \in \calH_m(\tau) } 
\frac{|\frac{1}{n}\sum_{i=1}^n \sigma_i f_m(x_i)|}{\Uns(f_m)} \right] 
+
\sum_{k= 1}^{\infty}
\EE\left[\sup_{f_m \in \calH_m(\tau z^k) \backslash \calH_m(\tau z^{k-1})} 
\frac{|\frac{1}{n}\sum_{i=1}^n \sigma_i f_m(x_i)|}{\Uns(f_m)} \right] \\
\leq 
&
C_{\sm}'  \frac{\frac{\tau^{1-\sm}\tilde{c}_{\sm}^{\sm}}{\sqrt{n}}}
{ 3 \frac{\cm^{-\sm}}{\sqrt{n}} \sm \cm} 
\vee 
\frac{\frac{C_1^{\frac{1-\sm}{1+\sm}} \tau^{\frac{(1-\sm)^2}{1+\sm}}\tilde{c}_{\sm}^{\frac{2\sm}{1+\sm}}}{n^{\frac{1}{1+\sm}}}}
{3\frac{\cm^{-\frac{\sm(3-\sm)}{1+\sm}}}{n^{\frac{1}{1+\sm}}} \sm\cm } \\
&+
\sum_{k=1}^\infty
C_{\sm}'  \frac{\frac{z^{k(1-\sm)}\tau^{1-\sm}\tilde{c}_{\sm}^{\sm}}{\sqrt{n}}}
{ 3 \frac{\cm^{-\sm}}{\sqrt{n}} \tau z^{k-1}} 
\vee 
\frac{\frac{C_1^{\frac{1-\sm}{1+\sm}} z^{k\frac{(1-\sm)^2}{1+\sm}} \tau^{\frac{(1-\sm)^2}{1+\sm}}\tilde{c}_{\sm}^{\frac{2\sm}{1+\sm}}}{n^{\frac{1}{1+\sm}}}}
{3\frac{\cm^{-\frac{\sm(3-\sm)}{1+\sm}}}{n^{\frac{1}{1+\sm}}} \tau z^{k-1} } \\
\leq  
&
\frac{C_{\sm}' }{3} \left(\sm^{-\sm} \tilde{c}_{\sm}^{\sm} \vee \sm^{-3\sm} C_1^{\frac{1-\sm}{1+\sm}} \tilde{c}_{\sm}^{\frac{2\sm}{1+\sm}}\right) 
\left( 1 +  \sum_{k=1}^\infty z^{1 - k\sm} \vee z^{1 - k\frac{\sm(3-\sm)}{1+\sm}} \right) \\
=
&
\frac{C_{\sm}' \sm^{-3\sm}}{3}  \left(\tilde{c}_{\sm}^{\sm} \vee C_1^{\frac{1-\sm}{1+\sm}} \tilde{c}_{\sm}^{\frac{2\sm}{1+\sm}}\right) 
\left(1 +  \frac{z^{1-\sm}}{1-z^{-\sm}}\vee \frac{z^{1-\frac{\sm(3-\sm)}{1+\sm}}}{1-z^{-\frac{\sm(3-\sm)}{1+\sm}}} \right) \\
\leq
&
9 C_{\sm}'   \left(\tilde{c}_{\sm}^{\sm} \vee C_1^{\frac{1-\sm}{1+\sm}} \tilde{c}_{\sm}^{\frac{2\sm}{1+\sm}}\right) 
\left(1 +  \frac{z^{1-\sm}}{1-z^{-\sm}}\vee \frac{z^{1-\frac{\sm(3-\sm)}{1+\sm}}}{1-z^{-\frac{\sm(3-\sm)}{1+\sm}}} \right),
\end{align*}
where we used $\sm^{-\sm} \leq 3$ for $0<\sm$ in the last line.
Thus by setting, $C_{\sm} = 
9 C_{\sm}'   \left(\tilde{c}_{\sm}^{\sm} \vee C_1^{\frac{1-\sm}{1+\sm}} \tilde{c}_{\sm}^{\frac{2\sm}{1+\sm}}\right) 
\left(1 +  \frac{z^{1-\sm}}{1-z^{-\sm}}\vee \frac{z^{1-\frac{\sm(3-\sm)}{1+\sm}}}{1-z^{-\frac{\sm(3-\sm)}{1+\sm}}} \right)$,
we obtain the assertion.
\end{proof}

This lemma immediately gives the following corollary.
\begin{Corollary}
\label{cor:basicuniformcor}
Under the Spectral Assumption (Assumption \ref{eq:specass}) and the Embedded Assumption (Assumption \ref{ass:linfbound}),  
there exists a constant $C_{\sm}$ depending only on $\sm,c$ and $C_1$ such that 
\begin{align*}
\EE\left[\sup_{f_m \in \calH_m} 
\frac{|\frac{1}{n}\sum_{i=1}^n \sigma_i f_m(x_i)|}
{\Uns(f_m)}\right] 
\leq C_{\sm}. 
\end{align*}
\end{Corollary}
\begin{proof}
By dividing the denominator and the numerator by the RKHS norm $\|f_m\hnorm{m}$, we have
\begin{align*}
&\EE\left[\sup_{f_m \in \calH_m}  
\frac{|\frac{1}{n}\sum_{i=1}^n \sigma_i f_m(x_i)|}
{\Uns(f_m)}\right]  \\
= &
\EE\left[\sup_{f_m \in \calH_m} 
\frac{|\frac{1}{n}\sum_{i=1}^n \sigma_i f_m(x_i)|/\|f_m\hnorm{m}}
{\Uns(f_m)/\|f_m\hnorm{m}}\right] \\
= &
\EE\left[\sup_{f_m \in \calH_m} 
\frac{|\frac{1}{n}\sum_{i=1}^n \sigma_i f_m(x_i)/\|f_m\hnorm{m}|}
{\Uns(f_m/\|f_m\hnorm{m})}\right] \\
= &
\EE\left[\sup_{f_m \in \calH_m: \|f_m \hnorm{m}=1} 
\frac{|\frac{1}{n}\sum_{i=1}^n \sigma_i f_m(x_i)|}
{\Uns(f_m)}\right] \\
\leq & C_{\sm}. 
~~~~~~(\because \text{Lemma \ref{lemm:uniformratiobound}})
\end{align*}

\end{proof}

\begin{Lemma}
\label{lemm:basicuniformlemm}
If $\frac{\log(M)}{\sqrt{n}}\leq 1$, then under the Spectral Assumption (Assumption \ref{eq:specass}) and the Embedded Assumption (Assumption \ref{ass:linfbound})  
there exists a constant $\tilde{C}_*$ depending only on $\{s_m\}_{m=1}^M$, $c$, $C_1$ such that 
\begin{align*}
\EE\left[\max_m \sup_{f_m \in \calH_m} 
\frac{|\frac{1}{n}\sum_{i=1}^n \sigma_i f_m(x_i)|}{\Uns(f_m)} \right] 
\leq \tilde{C}_*. 
\end{align*}
\end{Lemma}

\begin{proof} [Proof of Lemma \ref{lemm:basicuniformlemm}]
First notice that the $\LPi$-norm and the $\infty$-norm of $\frac{\sigma_i f_m(x_i)}{\Uns(f_m)}$ can be evaluated by  
\begin{align}
&\left\| \frac{\sigma_i f_m(x_i)}{\Uns(f_m)} \right\|_{\LPi} = \frac{\left\|f_m\right\|_{\LPi}}{\Uns(f_m)} 
\leq
 \frac{\|f_m\|_{\LPi}}{\sqrt{\frac{\log(M)}{n}} \|f_m\|_{\LPi}}
\leq \sqrt{\frac{n}{\log(M)}},
\label{eq:ratioLtwonormBound}
\\
&
\left\| \frac{\sigma_i f_m(x_i)}{\Uns(f_m)} \right\|_{\infty} =  \frac{\| f_m \|_{\infty}}{\Uns(f_m)} 
\leq  \frac{C_1 \|f_m \|_{\LPi}^{1-\sm} \|f_m \hnorm{m}^{\sm} }{\Uns(f_m)} \leq \frac{C_1}{3} \sqrt{n}\leq C_1 \sqrt{n},
\label{eq:ratioinfnormBound}
\end{align}
where the second line is shown by using the relation \eqref{eq:sdecomptwo}.
Let $C_* := \max_m C_{\sm}$ where $C_{\sm}$ is the constant appeared in Lemma \ref{lemm:uniformratiobound}.
Thus Talagrand's inequality and Corollary \ref{cor:basicuniformcor} imply
\begin{align*}
&P\left(\max_m \sup_{f_m \in \calH_m} 
\frac{|\frac{1}{n}\sum_{i=1}^n \sigma_i f_m(x_i)|}{\Uns(f_m)}
\geq
K\left[C_* + \sqrt{\frac{t}{\log(M)}} + \frac{C_1 t}{\sqrt{n}}\right]
\right) \notag \\
\leq 
&
\sum_{m=1}^M 
P\left(\sup_{f_m \in \calH_m} 
\frac{|\frac{1}{n}\sum_{i=1}^n \sigma_i f_m(x_i)|}{\Uns(f_m)}
\geq
K\left[C_* + \sqrt{\frac{t}{\log(M)}} + \frac{C_1 t}{\sqrt{n}}\right]
\right) \notag 
\\ 
\leq
&
\sum_{m=1}^M 
P\left(\sup_{f_m \in \calH_m} 
\frac{|\frac{1}{n}\sum_{i=1}^n \sigma_i f_m(x_i)|}{\Uns(f_m)}
\geq
K\left[C_{\sm} + \sqrt{\frac{t}{\log(M)}} + \frac{C_1 t}{\sqrt{n}}\right]
\right) \notag 
\\ 
\leq
&
M e^{-t}.
\end{align*}
By setting $t \leftarrow t + \log(M)$, we obtain 
\begin{align*}
&P\left(\max_m \sup_{f_m \in \calH_m} 
\frac{|\frac{1}{n}\sum_{i=1}^n \sigma_i f_m(x_i)|}{\Uns(f_m)}
\geq
K\left[C_* + \sqrt{\frac{t + \log(M)}{\log(M)}} + \frac{C_1 (t + \log(M))}{\sqrt{n}}\right]
\right) \leq e^{-t}
\notag 
\end{align*}
for all $t\geq 0$.
Consequently the expectation of the $\max$-$\sup$ term can be bounded as 
\begin{align*}
&
\EE\left[\max_m 
\sup_{f_m \in \calH_m} 
\frac{|\frac{1}{n}\sum_{i=1}^n \sigma_i f_m(x_i)|}{\Uns(f_m)} \right] \notag \\
\leq 
&
K\left[C_* + 1 + \frac{C_1\log(M)}{\sqrt{n}}\right]
+ \int_{0}^{\infty} K\left[C_* + \sqrt{\frac{t + 1 + \log(M)}{\log(M)}} + \frac{C_1(t+1+\log(M))}{\sqrt{n}}\right]
e^{-t} \dd t \notag \\
\leq &
2K\left[C_* + \sqrt{2} + \sqrt{\frac{\pi}{4 \log(M)}} + \frac{C_1(2+\log(M))}{\sqrt{n}}\right] \leq \tilde{C}_*, 
\notag 
\end{align*}
where we used $\sqrt{t + 1 + \log(M)} \leq \sqrt{t} + \sqrt{1+\log(M)}$ and $\int_{0}^\infty \sqrt{t} e^{-t} \dd t = \sqrt{\frac{\pi}{4}}$, $\frac{\log(M)}{\sqrt{n}} \leq 1$, and 
$\tilde{C}_* = 2K[C_* + \sqrt{2} +\sqrt{\frac{\pi}{4}}+3C_1]$.
\end{proof}

\begin{Lemma}
\label{lemm:uniformRatioBoundOnM}
Suppose the Basic Assumption (Assumption \ref{ass:BasicAss}),
 the Spectral Assumption (Assumption \ref{eq:specass}) and the Embedded Assumption (Assumption \ref{ass:linfbound}) hold.
Define $\bar{\phi} = K L \left[2 \tilde{C}_* +1 + C_1 \right]$.
If $\frac{\log(M)}{\sqrt{n}}\leq 1$, then 
the following holds 
\begin{align*}
&P\left(\max_m \sup_{f_m \in \calH_m} 
\frac{|\frac{1}{n}\sum_{i=1}^n \epsilon_i f_m(x_i)|}{\Uns(f_m)}
\geq 
\bar{\phi}
\eta(t) \right) \leq
 e^{-t}. 
\end{align*}
\end{Lemma}
\begin{proof} [Proof of Lemma \ref{lemm:uniformRatioBoundOnM}]
By the contraction inequality \cite[Theorem 4.12]{Book:Ledoux+Talagrand:1991} and Lemma \ref{lemm:basicuniformlemm},
we have 
\begin{align*}
&
\EE\left[\max_m \sup_{f_m \in \calH_m} 
\frac{|\frac{1}{n}\sum_{i=1}^n \epsilon_i f_m(x_i)|}{\Uns(f_m)}\right] 
\leq 
2 \EE\left[\max_m \sup_{f_m \in \calH_m} 
\frac{|\frac{1}{n}\sum_{i=1}^n \sigma_i \epsilon_i f_m(x_i)|}{\Uns(f_m)}\right] \leq 
2 L \tilde{C}_*,
\end{align*}
where we used $\epsilon_i \leq L$ (Basic Assumption).
Using this and \Eqref{eq:ratioLtwonormBound} and \Eqref{eq:ratioinfnormBound},
Talgrand's inequality gives 
\begin{align*}
&P\left(\max_m \sup_{f_m \in \calH_m} 
\frac{|\frac{1}{n}\sum_{i=1}^n \epsilon_i f_m(x_i)|}{\Uns(f_m)}
\geq K L \left[
2 \tilde{C}_* 
+
\sqrt{t} + \frac{C_1 t}{\sqrt{n}} \right]
\right) 
\leq
 e^{-t}. 
\end{align*}
Thus we have 
\begin{align*}
&P\left(\max_m \sup_{f_m \in \calH_m} 
\frac{|\frac{1}{n}\sum_{i=1}^n \epsilon_i f_m(x_i)|}{\Uns(f_m)}
\geq K L \left[
2 \tilde{C}_* 
+
1 + C_1 \right]
\max\left(1,\sqrt{t},\frac{t}{\sqrt{n}}\right)
\right) 
\leq
 e^{-t}. 
\end{align*}
Therefore by the definition of $\bar{\phi}$ and $\eta(t)$, we obtain the assertion.

\end{proof}

\begin{Lemma}
\label{th:SquareBound}
Suppose the Basic Assumption (Assumption \ref{ass:BasicAss}),
 the Spectral Assumption (Assumption \ref{eq:specass}) and the Embedded Assumption (Assumption \ref{ass:linfbound}) hold.
Let $\bar{\phi}'=K[2C_1 \tilde{C}_*   +  C_1 + C_1^2]$. 
Then, 
if $\frac{\log(M)}{\sqrt{n}} \leq 1$,
we have for all $t\geq 0$
\begin{align*}
\left| \textstyle \left\|\sum_{m=1}^M f_m \right\|_n^2 - \left\|\sum_{m=1}^M f_m \right\|_{\LPi}^2 \right|  \leq 
\phi'
\sqrt{n} \left(
\sum_{m=1}^M \Uns(f_m) \right)^2  \eta(t),~~
\end{align*}
for all  $f_m \in \calH_m~(m=1,\dots,M)$
with probability $1- \exp( - t )$. 
\end{Lemma}
\begin{proof} [Proof of Lemma \ref{th:SquareBound}]
\begin{align}
&\EE\left[\sup_{f_m \in \calH_m} \frac{\left| \textstyle \left\|\sum_{m=1}^M f_m \right\|_n^2 - \left\|\sum_{m=1}^M f_m \right\|_{\LPi}^2 \right| }{
\left(\sum_{m=1}^M \Uns(f_m) \right)^2}  \right] \notag \\
\leq &  
2 \EE\left[ \sup_{f_m \in \calH_m} \frac{\textstyle  \left| \frac{1}{n} \sum_{i=1}^n \sigma_i ( \sum_{m=1}^M f_m(x_i))^2   \right| }{
\left(\sum_{m=1}^M \Uns(f_m) \right)^2}  \right] 
\notag \\
\leq &
\sup_{f_m \in \calH_m} 
\frac{\textstyle  \left\| \sum_{m=1}^M f_m \right\|_{\infty} }{
\sum_{m=1}^M \Uns(f_m) } \times
2 \EE\left[ 
\sup_{f_m \in \calH_m} 
\frac{\textstyle  \left| \frac{1}{n} \sum_{i=1}^n \sigma_i ( \sum_{m=1}^M f_m(x_i))   \right| }{
\sum_{m=1}^M  \Uns(f_m) }  \right],
\label{eq:squareUpperBoundtmp}
\end{align} 
where we used the contraction inequality in the last line \cite[Theorem 4.12]{Book:Ledoux+Talagrand:1991}.
Thus using \Eqref{eq:ratioinfnormBound}, the RHS of the inequality \eqref{eq:squareUpperBoundtmp} can be bounded as 
\begin{align*}
&2C_1 \sqrt{n} \EE\left[ \sup_{f_m \in \calH_m} 
\frac{\textstyle  \left| \frac{1}{n} \sum_{i=1}^n \sigma_i ( \sum_{m=1}^M f_m(x_i))   \right| }{
\sum_{m=1}^M \Uns(f_m) }  \right] \\
\leq & 
2 C_1 \sqrt{n} \EE\left[ \sup_{f_m \in \calH_m} \max_m 
\frac{\textstyle  \left| \frac{1}{n} \sum_{i=1}^n \sigma_i f_m(x_i)   \right| }{
\Uns(f_m)  }  \right],
\end{align*}
where we used the relation 
\begin{equation}
\frac{\sum_{m} a_m}{\sum_m b_m} \leq \max_m\left(\frac{a_m}{b_m}\right)
\label{eq:maxsuprelation}
\end{equation}
for all $a_m \geq 0$ and $b_m \geq 0$ with a convention $\frac{0}{0}=0$.
By Lemma \ref{lemm:basicuniformlemm}, 
the right hand side is upper bounded by $2 C_1 \sqrt{n} \tilde{C}_*$.
Here we again apply Talagrand's concentration inequality, then we have
\begin{align}
&P\left(  \sup_{f_m \in \calH_m} \frac{ \left| \textstyle \left\|\sum_{m=1}^M f_m \right\|_n^2 - \left\|\sum_{m=1}^M f_m \right\|_{\LPi}^2 \right| }
{\left(\sum_{m=1}^M \Uns(f_m) \right)^2} 
\geq K\left[2C_1 \tilde{C}_* \sqrt{n}  + \sqrt{tn} C_1  + C_1^2 t   \right] \right) \leq e^{-t},
\notag
\end{align}
where we substituted the following upper bounds of $B$ and $U$.
\begin{align*}
B\leq &\sup_{f_m \in \calH_m} \EE \left[ \left( \frac{(\sum_{m=1}^M f_m)^2 }
{\left(\sum_{m=1}^M \Uns(f_m)\right)^2}\right)^2 \right] \\
\leq & 
\sup_{f_m \in \calH_m} \EE \left[ \frac{(\sum_{m=1}^M f_m)^2 } 
{\left( \sum_{m=1}^M  \Uns(f_m) \right)^2} 
\frac{(\|\sum_{m=1}^M f_m\|_{\infty})^2 } 
{ \left( \sum_{m=1}^M \Uns(f_m)\right)^2}  \right] \\
\mathop{\leq}^{\text{\eqref{eq:ratioinfnormBound}}} & 
\sup_{f_m \in \calH_m}  \frac{\left(\sum_{m=1}^M \|f_m\|_{\LPi}\right)^2 } 
{\left( \sum_{m=1}^M  \Uns(f_m) \right)^2} 
\frac{(\sum_{m=1}^M C_1 \sqrt{n} \Uns(f_m))^2 } 
{ \left( \sum_{m=1}^M \Uns(f_m)\right)^2}  \\
\mathop{\leq}^{\text{\eqref{eq:ratioLtwonormBound}}} & C_1^2 n^2 \frac{1}{\log(M)} \leq C_1^2 n^2,
\end{align*}
where in the second inequality we used the relation 
$$
\textstyle 
\EE\left[\left(\sum_{m=1}^M f_m\right)^2\right] = \EE\left[ \sum_{m,m'=1}^M f_m f_{m'}\right] 
\leq \sum_{m,m'=1}^M \|f_m\|_{\LPi} \|f_{m'}\|_{\LPi} = (\sum_{m=1}^M \|f_m\|_{\LPi})^2
$$
and in the third and forth inequality we used  
\Eqref{eq:ratioinfnormBound} and \Eqref{eq:ratioLtwonormBound} with Eq.\eqref{eq:maxsuprelation}
respectively.
Here we again use \Eqref{eq:ratioLtwonormBound} with Eq.\eqref{eq:maxsuprelation} to obtain 
\begin{align*}
U = &\sup_{f_m \in \calH_m} \left\| \frac{(\sum_{m=1}^M f_m)^2 }
{\left(\sum_{m=1}^M \Uns(f_m) \right)^2}\right\|_{\infty} 
\leq 
C_1^2 n.
\end{align*}
Therefore 
the above inequality implies 
the following inequality 
\begin{align}
&\sup_{f_m \in \calH_m} \frac{ \left| \textstyle \left\|\sum_{m=1}^M f_m \right\|_n^2 - \left\|\sum_{m=1}^M f_m \right\|_{\LPi}^2 \right| }
{\left(\sum_{m=1}^M \Uns(f_m) \right)^2}  
\leq 
K\left[2C_1 \tilde{C}_s  + C_1 + C_1^2  \right]\sqrt{n}\max(1,\sqrt{t},t/\sqrt{n}),\notag
\end{align} 
with probability $1 - \exp( - t)$.
Remind $\bar{\phi}' = K\left[2C_1 \tilde{C}_*   +  C_1 + C_1^2  \right]$, then we obtain the assertion. 
\end{proof}

\section{Proof of Theorem \ref{th:LowerBounds} (minimax learning rate)}
\label{sec:proofOfMinimax}

Let {\it the $\delta$-packing number} $Q(\delta,\calH,\LPi)$ of a function class $\calH$ be the largest number of functions $\{f_1, \dots, f_Q \} \subseteq \calH$
such that $\|f_i - f_j\|_{\LPi} \geq \delta$ for all $i\neq j$.

\begin{proof} [Proof of Theorem \ref{th:LowerBounds}]
The proof utilizes the techniques developed by \cite{NIPS:Raskutti+Martin:2009,arXiv:Raskutti+Martin:2010} that applied the information theoretic technique developed by 
\cite{AS:Yang+Barron:99} to the MKL settings.
To simplify the notation, we write $\calF := \calHpsi(R)$, $N(\varepsilon,\calH) := N(\varepsilon,\calH,\LPi)$ and $Q(\varepsilon,\calH) := Q(\varepsilon,\calH,\LPi)$.
It can be easily shown that $Q(2\varepsilon,\calF) \leq N(2\varepsilon,\calF) \leq Q(\varepsilon,\calF)$.
Here due to Theorem 15 of \cite{COLT:Steinwart+etal:2009},   
Assumption \ref{eq:specass2} yields
\begin{align}
\label{eq:coveringUpperLowerBounds}
\log N(\varepsilon, \tilde{\calH}(1)) \sim \varepsilon^{-2 \smm{}}.
\end{align}

We utilize the following inequality given by Lemma 3 of \cite{NIPS:Raskutti+Martin:2009}:
\[
\min_{\hat{f}} \max_{f^* \in \calHpsi(R_p)} \EE \|\hat{f} - f^* \|_{\LPi}^2 \geq 
\frac{\delta_n^2}{4}\left(1 - \frac{\log N(\varepsilon_n,\calF) + n \varepsilon_n^2/2\sigma^2 + \log 2}{\log Q(\delta_n,\calF)} \right).
\]

First we show the assertion for the $\ell_{\infty}$-norm ball:
$\calHpsi(R) = \calHl{\infty}(R)  := \left\{f = \sum_{m=1}^M f_m \bmid \max_{1\leq m \leq M} \|f_m\hnorm{m} \leq R \right\}.$ 
In this situation, there is a constant $C$ that depends only $s$ such that 
\begin{align*}
\log Q(\delta,\calF) \geq C M \log Q(\delta/\sqrt{M},\repH(R)),~~~
\log N(\varepsilon,\calF) \leq M \log N(\varepsilon/\sqrt{M},\repH(R)),
\end{align*}
(this is shown in Lemma 5 of \cite{arXiv:Raskutti+Martin:2010}, but we give the proof in Lemma \ref{lemm:QlogMbound} for completeness).
Using this expression, the minimax-learning rate is bounded as 
\[
\min_{\hat{f}} \max_{f^* \in \calHlp(R_p)} \EE \|\hat{f} - f^* \|_{\LPi}^2 \geq 
\frac{\delta_n^2}{4}\left(1 - \frac{ M \log N(\varepsilon_n/\sqrt{M},\repH(R)) + n \varepsilon_n^2/2\sigma^2 + \log 2}{
C M \log Q(\delta_n/\sqrt{M},\repH(R))} \right).
\]
Here we choose $\varepsilon_n$ and $\delta_n$ to satisfy the following relations: 
\begin{align}
&\frac{n}{2\sigma^2} \varepsilon_n^2 \leq M \log N\left(\varepsilon_n/\sqrt{M},\repH(R)\right), \label{eq:epssqMN} \\
&M \log N\left(\varepsilon_n/\sqrt{M},\repH(R)\right) \geq \log 2, \label{eq:2leqMlogN} \\
&4 \log N\left(\epsilon_n/\sqrt{M},\repH(R)\right) \leq C \log Q\left(\delta_n/\sqrt{M},\repH(R)\right). \label{eq:NleqM} 
\end{align}
With $\varepsilon_n$ and $\delta_n$ that satisfy the above relations \eqref{eq:epssqMN} and \eqref{eq:NleqM}, we have 
\begin{align}
\min_{\hat{f}} \max_{f^* \in \calHlp(R_p)} \EE \|\hat{f} - f^* \|_{\LPi}^2 \geq \frac{\delta_n^2}{16}.
\label{eq:basicMinimaxBound}
\end{align}
By \Eqref{eq:coveringUpperLowerBounds},
the relation \eqref{eq:epssqMN} can be rewritten as 
$$
\frac{n}{2\sigma^2} \varepsilon_n^2 \leq C M \left(\frac{\varepsilon_n}{R \sqrt{M}} \right)^{-2s}.
$$
It is sufficient to impose
\begin{equation}
\varepsilon_n^2 \leq C n^{-\frac{1}{1+s}} M R^{\frac{2s}{1+s}},
\label{eq:varepsilonnMRbound}
\end{equation}
with a constant $C$.
Since we have assumed that $n > \frac{\bar{c}^2  M^2}{R^2 \|\boldone\|_{\psi^*}^2}$ $(=\frac{1}{R^2}~\text{for $\|\cdot\|_{\psi}=\|\cdot\|_{\ell_\infty}$})$, 
the conditions \eqref{eq:2leqMlogN} can be satisfied 
if the constant $C$ in \Eqref{eq:varepsilonnMRbound} is taken sufficiently small
so that we have 
\begin{equation}
\log 2 \leq \log N(\varepsilon_n/\sqrt{M},\repH(R)) \sim \left(\frac{\varepsilon_n}{R \sqrt{M}} \right)^{-2s}.
\label{eq:logNandlog2inequ}
\end{equation}
The relation \eqref{eq:NleqM} can be satisfied by taking $\delta_n = c\varepsilon_n$ with an appropriately chosen constant $c$.
Thus \Eqref{eq:basicMinimaxBound} gives 
\begin{align}
\min_{\hat{f}} \max_{f^* \in \calHlp(R_p)} \EE \|\hat{f} - f^* \|_{\LPi}^2 \geq C n^{-\frac{1}{1+s}} M R^{\frac{2s}{1+s}},
\label{eq:MinimaxBoundInfty}
\end{align}
with a constant $C$. This gives the assertion for $p=\infty$.

Finally we show the assertion for general isotropic $\psi$-norm $\|\cdot\|_{\psi}$.  
To show that, we prove that $\calHl{\infty}(R\|\boldone\|_{\psi^*}/(\bar{c}M)) \subset \calHpsi(R)$. 
This is true if $\frac{R \|\boldone\|_{\psi^*}}{\bar{c}M} \boldone \in \calHpsi(R)$ because of the second condition of the definition \eqref{eq:defistropic} of isotropic property.
By the isotropic property, the $\psi$-norm of $\frac{R \|\boldone\|_{\psi^*}}{\bar{c}M} \boldone$ is bounded as
\begin{align*}
\left\|\frac{R \|\boldone\|_{\psi^*}}{\bar{c}M} \boldone \right\|_{\psi} & =  \frac{ R \|\boldone\|_{\psi^*}}{\bar{c}M} \left\| \boldone \right\|_{\psi} 
\mathop{\leq}^{\text{isotropic}} \frac{ R }{\bar{c}M} \bar{c} M = R.
\end{align*}
Thus we have $\frac{R \|\boldone\|_{\psi^*}}{\bar{c}M} \boldone \in \calHpsi(R)$ and thus $\calHl{\infty}(R\|\boldone\|_{\psi^*}/(\bar{c}M)) \subset \calHpsi(R)$.
Therefore we have 
\begin{align*}
\min_{\hat{f}} \max_{f^* \in \calHpsi(R)} \EE \|\hat{f} - f^* \|_{\LPi}^2 & \geq \min_{\hat{f}} \max_{f^* \in \calHl{\infty}(R\|\boldone\|_{\psi^*}/(\bar{c}M))} \EE \|\hat{f} - f^* \|_{\LPi}^2  \\
& \geq C n^{-\frac{1}{1+s}} M \left(\frac{R\|\boldone\|_{\psi^*}}{\bar{c}M}\right)^{\frac{2s}{1+s}},~~~~~~~~(\because \text{\Eqref{eq:MinimaxBoundInfty}}). 
\end{align*}
Note that due to the condition $n > \frac{\bar{c}^2 M^2}{R^2 \|\boldone\|_{\psi^*}^2}$,
\Eqref{eq:MinimaxBoundInfty} is still valid under the condition that 
$\frac{R\|\boldone\|_{\psi^*}}{\bar{c}M}$ is substituted into $R$ in \Eqref{eq:MinimaxBoundInfty}
(more precisely, \Eqref{eq:logNandlog2inequ} is valid).
Resetting $C \leftarrow C \bar{c}^{-\frac{2s}{1+s}}$, we obtain the assertion.
\end{proof}

\begin{Lemma}
\label{lemm:QlogMbound}
There is a constant $C$ such that
$$
\log Q(\delta,\calHl{\infty}(R)) \geq C M \log Q(\delta/\sqrt{M},\repH(R)),
$$
for sufficiently small $\delta$.
\end{Lemma}
\begin{proof}
The proof is analogous to that of Lemma 5 in \cite{arXiv:Raskutti+Martin:2010}. 
We describe the outline of the proof. 
Let $N = Q(\sqrt{2} \delta/\sqrt{M},\repH(R))$ and $\{f_m^1,\dots,f_m^N\}$ be a $\sqrt{2}\delta/\sqrt{M}$-packing of $\calH_m(R)$.
Then we can construct a function class $\Upsilon$ as 
$$
\Upsilon = \left\{f^{\boldsymbol{j}} = \sum_{m=1}^M f_m^{j_m} \mid \boldsymbol{j} = (j_1,\dots,j_M) \in \{1,\dots,N\}^M \right\}.
$$

We denote by $[N] := \{1,\dots,N\}$.
For two functions $f^{\boldsymbol{j}},f^{\boldsymbol{j}'} \in \Upsilon$, 
we have by the construction
$$
\|f^{\boldsymbol{j}} - f^{\boldsymbol{j}'}\|_{\LPi}^2 = \sum_{m=1}^M \|f_m^{j_m} - f_m^{j'_m}\|_{\LPi}^2 \geq \frac{2\delta^2}{M} \sum_{m=1}^M \boldsymbol{1}[j_m \neq j'_m].
$$
Thus, it suffices to construct a sufficiently large subset $A \subset [N]^M$ such that all different pairs $\boldsymbol{j},\boldsymbol{j}'\in A$ 
have at least $M/2$ of Hamming distance $d_H(\boldsymbol{j},\boldsymbol{j}') := \sum_{m=1}^M \boldsymbol{1}[j_m \neq j'_m]$.

Now we define $d_H(A,\boldsymbol{j}) := \min_{\boldsymbol{j}' \in A} d_H(\boldsymbol{j}',\boldsymbol{j})$.
If $|A|$ satisfies 
\begin{align}
\label{eq:basicAbound}
\left|\left\{\boldsymbol{j} \in [N]^M ~\Big|~ d_H(A,\boldsymbol{j})\leq \frac{M}{2} \right\}\right|
< |[N]^M|=N^M,
\end{align}
then there exists a member $\boldsymbol{j}' \in [N]^M$ such that $\boldsymbol{j}'$ is more than $\frac{M}{2}$ away from $A$ with respect to $d_H$, i.e. $d_H(A,\boldsymbol{j}')>\frac{M}{2}$.
That is, we can add $\boldsymbol{j}'$ to $A$ as long as \Eqref{eq:basicAbound} holds.
Now since 
\begin{align}
\left|\left\{\boldsymbol{j} \in [N]^M ~\Big|~ d_H(A,\boldsymbol{j})\leq \frac{M}{2} \right\}\right|
\leq |A|  {M \choose M/2}N^{M/2}, 
\end{align}
\Eqref{eq:basicAbound} holds as long as $A$ satisfies
$$
|A| \leq \frac{1}{2} \frac{N^M}{{M \choose M/2}N^{M/2}} =: Q^*.
$$
The logarithm of $Q^*$ can be evaluated as follows 
\begin{align*}
\log Q^* & = \log\left(\frac{1}{2} \frac{N^M}{{M \choose M/2}N^{M/2}}\right)  = M\log N - \log 2 - \log{M \choose M/2} - \frac{M}{2} \log N \\
&\geq  \frac{M}{2} \log N - \log 2 - \log 2^M \geq \frac{M}{2} \log \frac{N}{16}.
\end{align*}
There exists a constant $C$ 
such that $N = Q(\sqrt{2} \delta/\sqrt{M},\repH(R)) \geq C Q(\delta/\sqrt{M},\repH(R))$ because  $\log Q(\delta,\repH(R)) \sim \left(\frac{\delta}{R} \right)^{-2s}$.
Thus we obtain the assertion for sufficiently large $N$.
\end{proof}

\section{Proof of Technical Lemmas}
\label{sec:MiscLemmas}
\subsection{Proof of Lemma \ref{lem:smcmuniformbound}}
\label{sec:smcmuniformboundproof}

Remind that \Eqref{eq:simpleConveRatemin} gives 
\begin{align}
&\|\fhat - \fstar\|_{\LPi}^2 \notag \\
&= \mathcal{O}_p\Bigg(
\min_{\substack{ \{\cm\}_{m=1}^M: \\ \cm >0}}\Bigg\{
\alpha_1^2 + \beta_1^2 + 
\left[ \left(\frac{\alpha_2}{\alpha_1}\right)^2 + \left(\frac{\beta_2}{\beta_1}\right)^2\right]  \|\fstar \|_{\psi}^2
+
\frac{M\log(M)}{n}\Bigg\} \Bigg).
\label{eq:simpleConveRateminTmp2}
\end{align}
We derive an upper bound of the right hand side by adding a constraint $\cm = \cmm{}~(\forall m)$.
Since $\sm = \smm{}~(\forall m)$,
under the constraint  $\cm = \cmm{}~(\forall m)$ we have 
\begin{align*}
&\frac{\alpha_2}{\alpha_1} = \frac{3 \frac{\smm{} \cmm{}^{1-\smm{}}}{\sqrt{n}} \left\| \boldone \right\|_{\psi^*} }{3 \sqrt{M \frac{\cmm{}^{-2\smm{}}}{n}}}
= \frac{1}{\sqrt{M}} \smm{} \cmm{}\left\| \boldone \right\|_{\psi^*}, \\
&\
\frac{\beta_2}{\beta_1} = \frac{3 \frac{\smm{} \cmm{}^{\frac{(1-\smm{})^2}{1+\smm{}}}}{n^{\frac{1}{1+\smm{}}}} \|\boldone\|_{\psi^*} }{3 
\sqrt{M \frac{\cmm{}^{-\frac{2\smm{}(3-\smm{})}{1+\smm{}}}}{n^{\frac{2}{1+\smm{}}}}}}
=  \frac{1}{\sqrt{M}} \smm{} \cmm{}\left\| \boldone \right\|_{\psi^*},
\end{align*}
Thus $\frac{\alpha_2}{\alpha_1} = \frac{\beta_2}{\beta_1}$, and \Eqref{eq:simpleConveRateminTmp2} becomes
\begin{align}
&\|\fhat - \fstar\|_{\LPi}^2 
= \calO_p\Bigg(
\min_{\substack{ \cmm{} >0, \\ \cm= \cmm{}}}\Bigg\{
\alpha_1^2 + \beta_1^2 + 
2
 \frac{1}{M} \smm{}^2 \cmm{}^2 \left\| \boldone \right\|_{\psi^*}^2 \|\fstar \|_{\psi}^2
+
\frac{M\log(M)}{n}\Bigg\} \Bigg).
\label{eq:simpleConveRateminTmp3}
\end{align}
By the definition, we see that the first two terms are monotonically decreasing function with respect to $\cmm{}$ and 
the third term is monotonically increasing function.
The minimum of the right hand side is attained by balancing 
$\alpha_1^2 + \beta_1^2$ and $2 \frac{1}{M} \smm{}^2 \cmm{}^2 \left\| \boldone \right\|_{\psi^*}^2 \|\fstar \|_{\psi}^2$.
Since $\alpha_1^2 + \beta_1^2 \leq 2 \max\left(\alpha_1^2, \beta_1^2 \right)$,
\Eqref{eq:simpleConveRateminTmp3} indicates that 
\begin{align}
&\|\fhat - \fstar\|_{\LPi}^2 
\leq \calO_p\Bigg(
\min_{\substack{ \cmm{} >0, \\ \cm= \cmm{}}}\Bigg\{
2 \max\left(\alpha_1^2, \beta_1^2 \right) + 
2
 \frac{1}{M} \smm{}^2 \cmm{}^2 \left\| \boldone \right\|_{\psi^*}^2 \|\fstar \|_{\psi}^2
+
\frac{M\log(M)}{n}\Bigg\} \Bigg).
\label{eq:simpleConveRateminTmp4}
\end{align}
To balance the first term and the second term,
we need to consider two situations:
$\alpha_1^2 = \frac{1}{M} \smm{}^2 \cmm{}^2 \left\| \boldone \right\|_{\psi^*}^2 \|\fstar \|_{\psi}^2$
or
$\beta_1^2 = \frac{1}{M} \smm{}^2 \cmm{}^2 \left\| \boldone \right\|_{\psi^*}^2 \|\fstar \|_{\psi}^2$.

First we balance the terms $\alpha_1^2$ and $\frac{1}{M} \smm{}^2 \cmm{}^2 \left\| \boldone \right\|_{\psi^*}^2 \|\fstar \|_{\psi}^2$ under the restriction 
that $\cm = \cmm{}~(\forall m)$:
\begin{align}
 & \alpha_1^2 = \frac{1}{M} \smm{}^2 \cmm{}^2 \left\| \boldone \right\|_{\psi^*}^2 \|\fstar \|_{\psi}^2 \notag
\\
\Leftrightarrow ~~& 9 M \frac{\cmm{}^{-2\smm{}}}{n} = \frac{1}{M} \smm{}^2 \cmm{}^2 \left\| \boldone \right\|_{\psi^*}^2 \|\fstar \|_{\psi}^2 \notag
\\
\Leftrightarrow ~~& \cmm{}^{-1} = (\smm{}/3)^{\frac{1}{1+\smm{}}}M^{-\frac{1}{1+\smm{}}} n^{\frac{1}{2(1+\smm{})}} (\|\boldone\|_{\psi^*} \|f^*\|_{\psi})^{\frac{1}{1+\smm{}}}.
\label{eq:rconcreteAlpha}
\end{align}
For this $\cmm{}$, we obtain
\begin{align}
&\alpha_1^2 
=
9 M \frac{\cmm{}^{-2\smm{}}}{n} \notag \\
=
& 9^{\frac{1}{1+\smm{}}} \smm{}^{\frac{2\smm{}}{1+\smm{}}} M^{1-\frac{2\smm{}}{1+\smm{}}} n^{-\frac{1}{1+\smm{}}}(\|\boldone\|_{\psi^*} \|f^*\|_{\psi})^{\frac{2\smm{}}{1+\smm{}}}
\leq 
9 M^{1-\frac{2\smm{}}{1+\smm{}}} n^{-\frac{1}{1+\smm{}}}(\|\boldone\|_{\psi^*} \|f^*\|_{\psi})^{\frac{2\smm{}}{1+\smm{}}},
\label{eq:Alpha1bound}
\end{align}
where we used $\smm{}^{\frac{2\smm{}}{1+\smm{}}} \leq 1$ and $9^{\frac{1}{1+\smm{}}} \leq 9$ in the last inequality.

Next we balance the terms $\beta_1^2$ and $\frac{1}{M} \smm{}^2 \cmm{}^2 \left\| \boldone \right\|_{\psi^*}^2 \|\fstar \|_{\psi}^2$ under the restriction 
that $\cm = \cmm{}~(\forall m)$:
\begin{align*}
 & \beta_1^2 = \frac{1}{M} \smm{}^2 \cmm{}^2 \left\| \boldone \right\|_{\psi^*}^2 \|\fstar \|_{\psi}^2 \\
\Leftrightarrow ~~&9 M \frac{\cmm{}^{-\frac{2\smm{}(3-\smm{})}{1+\smm{}}}}{n^{\frac{2}{1+\smm{}}}} = 
\frac{1}{M} \smm{}^2 \cmm{}^2 \left\| \boldone \right\|_{\psi^*}^2 \|\fstar \|_{\psi}^2 \\
\Leftrightarrow ~~& 
\cmm{}^{-1} = 
(\smm{}/3)^{\frac{1+\smm{}}{1+4\smm{} - \smm{}^2}}  M^{-\frac{1+\smm{}}{1+4\smm{} - \smm{}^2}} n^{\frac{1}{1+4\smm{} - \smm{}^2}} 
\left(\|\boldone\|_{\psi^*} \|f^*\|_{\psi}\right)^{\frac{1+\smm{}}{1+4\smm{} - \smm{}^2}}.
\end{align*}
For this $\cmm{}$, we obtain
\begin{align*}
&\beta_1^2 
= 9 M \frac{\cmm{}^{-\frac{2\smm{}(3-\smm{})}{1+\smm{}}}}{n^{\frac{2}{1+\smm{}}}}  \\
=&
9^{\frac{1+\smm{}}{1+4\smm{}-\smm{}^2}} \smm{}^{\frac{2\smm{}(3-\smm{})}{1+4\smm{} - \smm{}^2}} M^{-\frac{1-2\smm{} + \smm{}^2}{1+4\smm{}-\smm{}^2}} n^{-\frac{2}{1+4\smm{}-\smm{}^2}}(\|\boldone\|_{\psi^*} \|f^*\|_{\psi})^{\frac{2\smm{}(3-\smm{})}{1+4\smm{}-\smm{}^2}} \\
\leq &
9 M^{\frac{1-2\smm{} + \smm{}^2}{1+4\smm{}-\smm{}^2}} n^{-\frac{2}{1+4\smm{}-\smm{}^2}}(\|\boldone\|_{\psi^*} \|f^*\|_{\psi})^{\frac{2\smm{}(3-\smm{})}{1+4\smm{}-\smm{}^2}},
\end{align*}
where we used $\smm{}^{\frac{2\smm{}(3-\smm{})}{1+4\smm{} - \smm{}^2}} \leq 1$ and $9^{\frac{1+\smm{}}{1+4\smm{}-\smm{}^2}}\leq 9$ in the last inequality.

Therefore the right hand side of \Eqref{eq:simpleConveRateminTmp4} is further bounded as
\begin{align*}
&\|\fhat - \fstar\|_{\LPi}^2 \\
\leq 
\calO_p\Bigg(
&
4 \max\Bigg\{ 9 M^{1-\frac{2\smm{}}{1+\smm{}}} n^{-\frac{1}{1+\smm{}}}(\|\boldone\|_{\psi^*} \|f^*\|_{\psi})^{\frac{2\smm{}}{1+\smm{}}}, \\
&9 M^{\frac{1-2\smm{} + \smm{}^2}{1+4\smm{}-\smm{}^2}} n^{-\frac{2}{1+4\smm{}-\smm{}^2}}(\|\boldone\|_{\psi^*} \|f^*\|_{\psi})^{\frac{2\smm{}(3-\smm{})}{1+4\smm{}-\smm{}^2}} \Bigg\} + 
\frac{M\log(M)}{n} \Bigg) \\
= 
\calO_p\Bigg( &
 M^{1-\frac{2\smm{}}{1+\smm{}}} n^{-\frac{1}{1+\smm{}}}(\|\boldone\|_{\psi^*} \|f^*\|_{\psi})^{\frac{2\smm{}}{1+\smm{}}} + \\
&M^{\frac{ (1-\smm{})^2}{1+4\smm{}-\smm{}^2}} n^{-\frac{2}{1+4\smm{}-\smm{}^2}}(\|\boldone\|_{\psi^*} \|f^*\|_{\psi})^{\frac{2\smm{}(3-\smm{})}{1+4\smm{}-\smm{}^2}}  + 
\frac{M\log(M)}{n} \Bigg).
\end{align*}
%
Finally, if $n \geq  (\|\boldone\|_{\psi^*} \|\fstar\|_{\psi}/M)^{\frac{4s}{1-s}}$,
the first term of the right hand side of this bound is not less than the second term:
$$
 M^{1-\frac{2\smm{}}{1+\smm{}}} n^{-\frac{1}{1+\smm{}}}(\|\boldone\|_{\psi^*} \|f^*\|_{\psi})^{\frac{2\smm{}}{1+\smm{}}} \geq 
 M^{\frac{ (1-\smm{})^2}{1+4\smm{}-\smm{}^2}} n^{-\frac{2}{1+4\smm{}-\smm{}^2}}(\|\boldone\|_{\psi^*} \|f^*\|_{\psi})^{\frac{2\smm{}(3-\smm{})}{1+4\smm{}-\smm{}^2}}.
$$

More precisely, with $\cmm{}$ given in \Eqref{eq:rconcreteAlpha}, 
the upper bound \eqref{eq:Alpha1bound} of $\alpha_1$
gives that, for $n \geq  (\|\boldone\|_{\psi^*} \|\fstar\|_{\psi}/M)^{\frac{4s}{1-s}}$, we have
\begin{align*}
\sqrt{n}\max\left\{\alpha_1^2, \beta_1^2,\frac{M\log(M)}{n}\right\}
&
\leq \sqrt{n} 9  M^{1-\frac{2\smm{}}{1+\smm{}}} n^{-\frac{1}{1+\smm{}}}(\|\boldone\|_{\psi^*} \|f^*\|_{\psi})^{\frac{2\smm{}}{1+\smm{}}} \vee \frac{M\log(M)}{\sqrt{n}} \\
&
= 
9  \left(\frac{M}{\sqrt{n}}\right)^{\frac{1-\smm{}}{1+\smm{}}}
(\|\boldone\|_{\psi^*} \|f^*\|_{\psi})^{\frac{2\smm{}}{1+\smm{}}} \vee \frac{M\log(M)}{\sqrt{n}}.
\end{align*}
Thus by setting 
$\lambdaone 
= 18 M^{\frac{1-\smm{}}{1+\smm{}}} n^{-\frac{1}{1+\smm{}}}\|\boldone\|_{\psi^*}^{\frac{2\smm{}}{1+\smm{}}} \|f^*\|_{\psi}^{-\frac{2}{1+\smm{}}}
\geq \left(\frac{\alpha_2}{\alpha_1}\right)^2 + \left(\frac{\beta_2}{\beta_1}\right)^2$,
then 
Theorem \ref{th:convergencerateofLpMKL} gives that
for all $n$ and $t'$ that satisfy  $\frac{\log(M)}{\sqrt{n}}\leq 1$ and 
$\frac{4\phi}{\kminrho} 
\left\{
9  \left(\frac{M}{\sqrt{n}}\right)^{\frac{1-\smm{}}{1+\smm{}}}
(\|\boldone\|_{\psi^*} \|f^*\|_{\psi})^{\frac{2\smm{}}{1+\smm{}}} \vee \frac{M\log(M)}{\sqrt{n}}
 \right\}
\eta(t') \leq \frac{1}{12}$
and for all $t \geq 1$,
we have 
\begin{align}
\|\fhat - \fstar\|_{\LPi}^2 
\leq 
&
\frac{24 \eta(t)^2 \phi^2}{\kminrho} \left(
18 M^{1-\frac{2\smm{}}{1+\smm{}}} n^{-\frac{1}{1+\smm{}}}(\|\boldone\|_{\psi^*} \|f^*\|_{\psi})^{\frac{2\smm{}}{1+\smm{}}}
 + \frac{M\log(M)}{n} \right) \notag \\
&+
4 \times 18 M^{1-\frac{2\smm{}}{1+\smm{}}} n^{-\frac{1}{1+\smm{}}}(\|\boldone\|_{\psi^*} \|f^*\|_{\psi})^{\frac{2\smm{}}{1+\smm{}}} \\
\leq
&
C \eta(t)^2  \left(
M^{1-\frac{2\smm{}}{1+\smm{}}} n^{-\frac{1}{1+\smm{}}}(\|\boldone\|_{\psi^*} \|f^*\|_{\psi})^{\frac{2\smm{}}{1+\smm{}}}
 + \frac{M\log(M)}{n} \right), \notag
\end{align}
with probability $1- \exp(- t) - \exp(-t')$ where $C$ is a sufficiently large constant depending on $\phi$ and $\kminrho$.
Finally notice that 
the condition 
$\frac{4\phi}{\kminrho} 
\left\{
9  \left(\frac{M}{\sqrt{n}}\right)^{\frac{1-\smm{}}{1+\smm{}}}
(\|\boldone\|_{\psi^*} \|f^*\|_{\psi})^{\frac{2\smm{}}{1+\smm{}}} \vee \frac{M\log(M)}{\sqrt{n}}
 \right\}
\eta(t') \leq \frac{1}{12}$
automatically gives $\frac{\log(M)}{\sqrt{n}}\leq 1$, thus we can drop the condition $\frac{\log(M)}{\sqrt{n}}\leq 1$.
Then we obtain the assertion.

\subsection{Proof of Lemma \ref{lemm:dualmixednorm}}
\label{sec:dualofnorms}
We assume $1<p<\infty$ and $1<q < \infty$. 
The proof for the situations $p=1,\infty$ or $q=1,\infty$ is straight forward.
First applying H{\"o}lder's inequality twice, we obtain  
\begin{align*}
\langle \boldb, \bolda \rangle &= 
\sum_{j=1}^{M'} \sum_{k=1}^{M_j} b_{j,k} a_{j,k} \\
&\leq 
\sum_{j=1}^{M'} \left\{ \left(\sum_{k=1}^{M_j} |b_{j,k}|^{p^*}\right)^{\frac{1}{p^*}} \left(\sum_{k=1}^{M_j} |a_{j,k}|^{p}\right)^{\frac{1}{p}}\right\} ~~~
(\because \text{H{\"o}lder's inequality}) \\
&\leq
\left\{\sum_{j=1}^{M'}\left(\sum_{k=1}^{M_j} |b_{j,k}|^{p^*}\right)^{\frac{q^*}{p^*}}\right\}^{\frac{1}{q^*}} 
\left\{\sum_{j=1}^{M'}\left(\sum_{k=1}^{M_j} |a_{j,k}|^{p}\right)^{\frac{q}{p}}\right\}^{\frac{1}{q}}~~~
(\because \text{H{\"o}lder's inequality}).
\end{align*}
Therefore we obtain that 
\begin{align}
\|\boldb\|_{\psi^*} \leq \left\{ \sum_{j=1}^{M'} (\sum_{k=1}^{M_j} |b_{j,k}|^{p^*})^{\frac{q^*}{p^*}} \right\}^{\frac{1}{q^*}}.
\label{eq:abinnerupper}
\end{align}
On the other hand, if we set
$$
a_{j,k} = b_{j,k}^{\frac{1}{p-1}} 
\frac{(\sum_{k=1}^{M_j} b_{j,k}^{p^*})^{\frac{q^*}{p^*}-1}}
{\{\sum_{j'=1}^{M'}(\sum_{k=1}^{M_{j'}} b_{j',k}^{p^*})^{\frac{q^*}{p^*}}\}^{\frac{1}{q}}},
$$
then we have 
\begin{align*}
\|\bolda\|_{\psi} 
& 
= 
\left\{ 
\sum_{j=1}^{M'}
\left(\sum_{k=1}^{M_j} b_{j,k}^{\frac{p}{p-1}}\right)^{\frac{q}{p}} 
\left(\sum_{k=1}^{M_j} b_{j,k}^{p^*}\right)^{q (\frac{q^*}{p^*}-1)}
\right\}^{\frac{1}{q}}
\frac{1}
{\{\sum_{j'=1}^{M'}(\sum_{k=1}^{M_{j'}} b_{j',k}^{p^*})^{\frac{q^*}{p^*}}\}^{\frac{1}{q}}} \\
& 
= 
\left\{
\sum_{j=1}^{M'}
\left(\sum_{k=1}^{M_j} b_{j,k}^{\frac{p}{p-1}}\right)^{q\left(\frac{1}{p} - 1 + \frac{q^*}{p^*}\right)} 
\right\}^{\frac{1}{q}}
\frac{1}
{\{\sum_{j'=1}^{M'}(\sum_{k=1}^{M_{j'}} b_{j',k}^{p^*})^{\frac{q^*}{p^*}}\}^{\frac{1}{q}}} \\
&
= 
\left\{
\sum_{j=1}^{M'}
\left(\sum_{k=1}^{M_j} b_{j,k}^{\frac{p}{p-1}}\right)^{\frac{q^*}{q^*-1}\left(\frac{q^*-1}{p^*}\right)} 
\right\}^{\frac{1}{q}}
\frac{1}
{\{\sum_{j'=1}^{M'}(\sum_{k=1}^{M_{j'}} b_{j',k}^{p^*})^{\frac{q^*}{p^*}}\}^{\frac{1}{q}}} = 1,
\end{align*}
and 
\begin{align*}
\langle \bolda,\boldb \rangle
& 
= 
\sum_{j=1}^{M'}
\left\{
\left(
\sum_{k=1}^{M_j}
b_{j,k}^{1 + \frac{1}{p-1}} 
\right)
\left(\sum_{k=1}^{M_j} b_{j,k}^{p^*}\right)^{\frac{q^*}{p^*}-1}
\right\}
\frac{1}
{\{\sum_{j'=1}^{M'}(\sum_{k=1}^{M_{j'}} b_{j',k}^{p^*})^{\frac{q^*}{p^*}}\}^{\frac{1}{q}}} \\
& 
= 
\sum_{j=1}^{M'}
\left(\sum_{k=1}^{M_j} b_{j,k}^{p^*}\right)^{\frac{q^*}{p^*}}
\frac{1}
{\{\sum_{j'=1}^{M'}(\sum_{k=1}^{M_{j'}} b_{j',k}^{p^*})^{\frac{q^*}{p^*}}\}^{\frac{1}{q}}} \\
& 
= 
\left\{\sum_{j'=1}^{M'}\left(\sum_{k=1}^{M_{j'}} b_{j',k}^{p^*}\right)^{\frac{q^*}{p^*}}\right\}^{\frac{1}{q^*}}.
\end{align*}
Therefore we obtain 
\begin{align}
\| \boldb \|_{\psi^*} \geq 
\left\{\sum_{j'=1}^{M'}\left(\sum_{k=1}^{M_{j'}} b_{j',k}^{p^*}\right)^{\frac{q^*}{p^*}}\right\}^{\frac{1}{q^*}}.
\label{eq:abinnerlower}
\end{align}
Combining Eqs.\eqref{eq:abinnerlower},\eqref{eq:abinnerlower}, we have 
$
\| \boldb \|_{\psi^*} 
=
\left\{\sum_{j'=1}^{M'}\left(\sum_{k=1}^{M_{j'}} b_{j',k}^{p^*}\right)^{\frac{q^*}{p^*}}\right\}^{\frac{1}{q^*}}.
$
Thus we obtain the assertion.

\subsection{Proof of Lemma \ref{lem:inhomogeneousConvRate}}
Remind that 
$$
\alpha_1 = 3 \left(\frac{r_1^{-2s} + M-1}{n}\right)^{\frac{1}{2}},
\alpha_2 = 3 \frac{s r_1^{1-s}}{\sqrt{n}},
\beta_1 = 3 \bigg(\frac{r_1^{-\frac{2s(3-s)}{1+s}} + M-1}{n^{\frac{2}{1+s}}}\bigg)^{\frac{1}{2}},
\beta_2 = 3 \frac{s r_1^{\frac{(1-s)^2}{1+s}}}{n^{\frac{1}{1+s}}}.
$$
Thus we have 
\begin{align*}
\left(\frac{\alpha_2}{\alpha_1}\right)^2 = \frac{\frac{s^2 r_1^{2(1-s)}}{n}}{
\frac{r_1^{-2s} + M-1}{n}} 
\simeq \min \left\{s^2 r_1^2, \frac{s^2 r_1^{2(1-s)}}{M-1}\right\},
\end{align*}
and
\begin{align*}
\left(\frac{\beta_2}{\beta_1}\right)^2 = 
\frac{\frac{s^2 r_1^{\frac{2(1-s)^2}{1+s}}}{n^{\frac{2}{1+s}}}}
{\frac{r_1^{-\frac{2s(3-s)}{1+s}} + M-1}{n^{\frac{2}{1+s}}}} 
\simeq \min \left\{ s^2 r_1^2, \frac{s^2 r_1^{\frac{2(1-s)^2}{1+s}}}{M-1} \right\}.
\end{align*}

Suppose $r_1^{-2s} \geq M-1$ and $r_1^{-\frac{2s(3-s)}{1+s}} \geq M-1$, then
we have $\alpha_1^2 \simeq  r_1^{-2s} n^{-1}$,
$\beta_1^2 = r_1^{-\frac{2s(3-s)}{1+s}} n^{-\frac{2}{1+s}}$
$\left(\frac{\alpha_2}{\alpha_1}\right)^2 \simeq s^2 r_1^2$ and $\left(\frac{\beta_2}{\beta_1}\right)^2 \simeq s^2 r_1^2$.
Thus the minimization problem in \Eqref{eq:simpleConveRatemin} with the constraint for $r_1$
becomes 
\begin{align}
&
\min_{\substack{ \cmm{1} > 0: \\ r_1^{-2s} \geq M-1,~r_1^{-\frac{2s(3-s)}{1+s}} \geq M-1 }} \Bigg\{
\alpha_1^2 + \beta_1^2 + 
\left[ \! \left(\frac{\alpha_2}{\alpha_1}\right)^2 \! + \left(\frac{\beta_2}{\beta_1}\right)^2\right]  \|\fstar \|_{\psi}^2 
 \Bigg\} \notag \\
\simeq
&
\min_{\substack{ \cmm{1} > 0: \\ r_1^{-2s} \geq M-1,~r_1^{-\frac{2s(3-s)}{1+s}} \geq M-1 }}
\Bigg\{ r_1^{-2s} n^{-1} + r_1^{-\frac{2s(3-s)}{1+s}} n^{-\frac{2}{1+s}} + r_1^2 \|\fstar \|_{\psi}^2 \Bigg\}.
\label{eq:minInhomr1}
\end{align}
If we neglect the constraints 
$r_1^{-2s} \geq M-1$ and $r_1^{-\frac{2s(3-s)}{1+s}} \geq M-1$,
the minimum is attained at $r_1$ (up to a constant factor) that satisfies 
$
\max\{r_1^{-2s} n^{-1},r_1^{-\frac{2s(3-s)}{1+s}} n^{-\frac{2}{1+s}}\} = r_1^2 \|\fstar \|_{\psi}^2,
$
i.e.
$$
r_1 = \max\left\{n^{-\frac{1}{2(1+s)}} \|\fstar \|_{\psi}^{-\frac{1}{1+s}}, 
n^{-\frac{1}{1+4s-s^2}} \|\fstar \|_{\psi}^{-\frac{1+s}{1+4s-s^2}}\right\}.
$$
Therefore if $n \geq \|\fstar\|_{\psi}^{\frac{4s}{1-s}}$
(this is satisfied because $\|\fstar \|_{\ell_1} = M$, $\|\fstar \|_{\ell_{\infty}} = 1$ and $n \geq M^{\frac{4s}{1-s}}$ is imposed), 
then the minimum is attained at 
$r_1 = n^{-\frac{1}{2(1+s)}} \|\fstar \|_{\psi}^{-\frac{1}{1+s}}.$
Finally the condition $n \geq (M\log(M))^{\frac{1+s}{s}}$ yields that 
$r_1^{-2s} \geq M-1$ and $r_1^{-\frac{2s(3-s)}{1+s}} \geq M-1$
for $r_1 = n^{-\frac{1}{2(1+s)}} \|\fstar \|_{\psi}^{-\frac{1}{1+s}}$.
Therefore the constraints for $r_1$ in \Eqref{eq:minInhomr1} can be removed.
Summarizing the above discussions, we obtain  
\begin{align*}
\min_{\substack{ \{\cm\}_{m=1}^M: \\ \cm > 0 }}\! \Bigg\{
\alpha_1^2 + \beta_1^2 + 
\left[ \! \left(\frac{\alpha_2}{\alpha_1}\right)^2 \! + \left(\frac{\beta_2}{\beta_1}\right)^2\right]  \|\fstar \|_{\psi}^2 
 \Bigg\} 
\simeq
n^{-\frac{1}{1+s}} \|\fstar \|_{\psi}^{\frac{2s}{1+s}}.
\end{align*}

Thus we obtain the following convergence rates:
\begin{align*}
&
\|\fhat^{(1)} - \fstar\|_{\LPi}^2
=\mathcal{O}_p\left( 
n^{-\frac{1}{1+s}} M^{\frac{2\smm{}}{1+\smm{}}}
+ \frac{M\log(M)}{n}
\right), \\
&
\|\fhat^{(\infty)} - \fstar\|_{\LPi}^2
=\mathcal{O}_p\left(
n^{-\frac{1}{1+s}} 
+ \frac{M\log(M)}{n}
 \right).
\end{align*}
Now since $n \geq (M\log(M))^{\frac{1+s}{s}}$, the above convergence rates can be simplified as 
\begin{align*}
&
\|\fhat^{(1)} - \fstar\|_{\LPi}^2
=\mathcal{O}_p\left( 
n^{-\frac{1}{1+s}} M^{\frac{2\smm{}}{1+\smm{}}}
\right),~~ 
\|\fhat^{(\infty)} - \fstar\|_{\LPi}^2
=\mathcal{O}_p\left(
n^{-\frac{1}{1+s}} 
 \right).
\end{align*}

\subsection{Proof of Lemma \ref{lemm:localRademacherComplexity} 
(Derivation of Local Rademacher Complexity)}
\label{sec:DerivationLocalRademacher}

For $f\in \calHtot$, we define 
$$
\Unall(f) 
:=
\alpha_1 \frac{\|f \|_{\LPi}}{\sqrt{\kminrho}} +  
\alpha_2 \|f \|_{\psi} 
+ \beta_1 \frac{\|f \|_{\LPi}}{\sqrt{\kminrho}} +  
\beta_2 \|f \|_{\psi} 
+\sqrt{\frac{M\log(M)}{n}} \frac{\|f \|_{\LPi}}{\sqrt{\kminrho}}.
$$
Then 
by \Eqref{eq:Unbounds} we obtain 
\begin{align*}
\sum_{m=1}^M \Uns(f_m) 
\leq 
\Unall(f).
\end{align*}

We know that there exists a constant $\tilde{\phi}$ such that 
\begin{align}
&P\left(
\max_m \sup_{f_m \in \calH_m} 
\frac{|\frac{1}{n}\sum_{i=1}^n \sigma_i f_m(x_i)|}{\Uns(f_m)}
\geq 
\tilde{\phi}
\eta(t) \right) \leq
 e^{-t},
\label{eq:PStbound}
\end{align}
(see Lemma \ref{lemm:uniformRatioBoundOnM}).
Let $
\bar{\eta}(t) := \max\{\sqrt{t},t/n\}, 
$
and the event $\calS_t$ be 
$$
\calS_t := 
\left\{
\tilde{\phi}
\bar{\eta}(t)
\leq 
\max_m \sup_{f_m \in \calH_m} 
\frac{|\frac{1}{n}\sum_{i=1}^n \sigma_i f_m(x_i)|}{\Uns(f_m)}
\leq 
\tilde{\phi}
\bar{\eta}(t+1)\right\}.
$$
Then, by \Eqref{eq:PStbound},
we have $P(\calS_t) \leq e^{-t}$ for $t\geq 1$.
Using this relation, we obtain the following upper bound of the local Rademacher complexity:
\begin{align*}
  &R_n(\calH_{\psi}^{(r)}(R))  \\
= &\EE\left[ \sup_{f \in \calH_{\psi}^{(r)}(R)} \frac{1}{n} \sum_{i=1}^n \sigma_i f(x_i) \right] \\
= &\sum_{t=0}^{\infty} \EE\left[ \sup_{f \in \calH_{\psi}^{(r)}(R)} \frac{1}{n} \sum_{i=1}^n \sigma_i f(x_i) \mid \calS_t \right] P(\calS_t) \\
\leq &
\EE\left[ \sup_{f \in \calH_{\psi}^{(r)}(R)} \frac{1}{n} \sum_{i=1}^n \sigma_i f(x_i) \mid \calS_0 \right]
+ \sum_{t=1}^{\infty} \EE\left[ \sup_{f \in \calH_{\psi}^{(r)}(R)} \frac{1}{n} \sum_{i=1}^n \sigma_i f(x_i) \mid \calS_t \right] P(\calS_t) \\
\leq & 
\EE\left[ \sup_{f \in \calH_{\psi}^{(r)}(R)}  \sum_{m=1}^M \tilde{\phi} \Uns(f_m) \mid S_0 \right] 
+ \sum_{t=1}^{\infty} \EE\left[ \sup_{f \in \calH_{\psi}^{(r)}(R)} \sum_{m=1}^M \tilde{\phi} \eta(t+1) \Uns(f_m) \mid \calS_t \right] e^{-t} \\ 
\leq & 
\EE\left[ \sup_{f \in \calH_{\psi}^{(r)}(R)} \tilde{\phi} \Unall(f) \mid S_0 \right] 
+ \sum_{t=1}^{\infty} \EE\left[ \sup_{f \in \calH_{\psi}^{(r)}(R)} \tilde{\phi} \eta(t+1) \Unall(f) \mid \calS_t \right] e^{-t} \\ 
\leq &
\tilde{\phi}
\left(\alpha_1 \frac{r}{\sqrt{\kminrho}} +  
\alpha_2 R 
+ \beta_1 \frac{r}{\sqrt{\kminrho}} +  
\beta_2 R 
+\sqrt{\frac{M\log(M)}{n}} \frac{r}{\sqrt{\kminrho}}\right) 
\left(1+\sum_{t=1}^{\infty} \eta(t+1) e^{-t}\right).
\end{align*}
Since 
\begin{align*}
\sum_{t=1}^{\infty} \eta(t+1) e^{-t}
\leq \int_{t=1}^{\infty} \left(\sqrt{t+1} + \frac{t+1}{\sqrt{n}}\right)e^{-(t-1)} \dd t
\leq 5,
\end{align*}
we obtain 
\begin{align*}
R_n(\calH_{\psi}^{(r)}(R)) \leq 
6 
\tilde{\phi}
\left(\alpha_1 \frac{r}{\sqrt{\kminrho}} +  
\alpha_2 R 
+ \beta_1 \frac{r}{\sqrt{\kminrho}} +  
\beta_2 R 
+\sqrt{\frac{M\log(M)}{n}} \frac{r}{\sqrt{\kminrho}}\right).
\end{align*}
By re-setting $\tilde{\phi} \leftarrow 6\tilde{\phi}$, we obtain the local Rademacher complexity upper bound.

\vskip 0.2in
\bibliographystyle{abbrvnat}
\bibliography{main} 

\end{document}